\documentclass{article}
\usepackage{graphicx} 
\usepackage{geometry}
\usepackage{amsmath}
\usepackage{amsfonts}
\usepackage{amsthm}
\usepackage{mathtools}
\usepackage{ulem}
\usepackage{amssymb}
\usepackage{dsfont}
\usepackage[dvipsnames]{xcolor}
\usepackage{amscd}
\usepackage{booktabs}
\usepackage[numbers]{natbib}
\PassOptionsToPackage{hyphens}{url}\usepackage{hyperref}
\usepackage{comment}
\hypersetup{
	colorlinks,
	linkcolor={red!50!black},
	citecolor={blue!50!black},
	urlcolor={blue!80!black},
    breaklinks=true
}
\usepackage{appendix}
\usepackage[capitalize,nameinlink,noabbrev]{cleveref}
\crefname{equation}{}{}

\usepackage{tikz-cd}
\usepackage{tikz}
\usepackage{indentfirst}
\usetikzlibrary{arrows.meta,
            decorations.markings,cd}
\usepackage{pgfplots}
\usepackage{enumitem}
\usepackage{subcaption}
\usepackage[abs]{overpic}
\usepackage{float}
\usepackage{natbib}
\usepackage{newtxtext}
\usepackage{newtxmath}
\usepackage{placeins}
\usepackage{todonotes}
\usepackage{algorithm}
\usepackage{algpseudocode}
\AddToHook{cmd/appendix/before}{%
  \crefalias{section}{appendix}%
  \crefalias{subsection}{appendix}%
}

\makeatletter
\AtBeginDocument{%
  \expandafter\renewcommand\expandafter\subsection\expandafter
    {\expandafter\@fb@secFB\subsection}%
  \newcommand\@fb@secFB{\FloatBarrier
    \gdef\@fb@afterHHook{\@fb@topbarrier \gdef\@fb@afterHHook{}}}%
  \g@addto@macro\@afterheading{\@fb@afterHHook}%
  \gdef\@fb@afterHHook{}%
}
\makeatother

\DeclareMathOperator{\GW}{GW}
\DeclareMathOperator{\Cpl}{Cpl}
\newcommand{\lb}{\llbracket}
\newcommand{\rb}{\rrbracket}
\newcommand{\Proj}{{\textrm{Proj}}}
\newcommand{\cP}{{\mathcal{P}}}
\newcommand{\hyperomegaX}{{\omega_{X, X'}}}
\newcommand{\hyperomegaY}{{\omega_{Y, Y'}}}

\newtheorem{Thm}{Theorem}[section]
\newtheorem{Lem}{Lemma}
\theoremstyle{definition}
\newtheorem{Def}{Definition}[section]

\newtheorem{remark}{Remark}[section]
\newtheorem{example}{Example}[section]

\title{A Unified Geometric Framework for Developmental Analysis of Spatial Transcriptomic Data }
\usepackage{authblk}

\author[1]{Mary Chriselda Antony Oliver\thanks{These authors contributed equally to this work.}}
\author[2]{Kaitlyn Hohmeier$^{*}$}
\author[3]{Tuyen Tran$^{*}$}
\author[4]{Alejandra Castillo}
\author[2]{Caroline Moosm\"{u}ller}
\author[5]{Shiying Li\thanks{Corresponding author: sli82@unl.edu}}

\affil[1]{Department of Applied Mathematics and Theoretical Physics, University of Cambridge}
\affil[2]{Department of Mathematics, University of North Carolina at Chapel Hill}
\affil[3]{Department of Mathematics and Statistics, Loyola University Chicago}
\affil[4]{Department of Mathematics, Pomona College}
\affil[5]{Department of Mathematics, University of Nebraska-Lincoln}

\date{}

\begin{document}

\maketitle

\begin{abstract}
High-throughput single-cell and spatial transcriptomic technologies provide high-resolution snapshots of heterogeneous cellular states, but their destructive nature prevents repeated measurements of the same cells over time. Consequently, temporal and spatial dynamics must be inferred from independently sampled, unaligned cell populations, making it challenging to reconstruct developmental trajectories.
Optimal transport (OT) offers a geometric framework for aligning cell populations and inferring developmental trajectories, but many existing approaches focus on modeling the evolution of distributions of cells in gene expression space rather than the relational structure encoded by gene expression networks. To address this limitation, we introduce a geometric framework for analyzing the spatiotemporal evolution of gene expression networks through embeddings in Gromov--Wasserstein (GW) space. By representing each developmental stage as a graph combining gene expression and spatial proximity, our approach enables comparisons of network structure across time, continuous interpolation between developmental stages via GW geodesics, and quantification of network-level changes using Ollivier-Ricci curvature. We evaluate our framework on a spatiotemporal transcriptomic \textit{Drosophila} dataset and show that GW geodesic interpolations reproduce main trends in curvature dynamics observed in empirical gene expression networks. Agreement with higher-order Co-Optimal Transport (COOT) distances, which jointly represent spatial and temporal information, further validates the framework and  suggests that hypernetwork representations successfully record salient biological changes across time. In general, our approach provides a unified geometric approach to study dynamically evolving biological networks.


\end{abstract}


\section{Introduction}
Modern single-cell genomic technologies provide pathways for multimodal profiling of cells in temporal and spatial dimensions. However, due to experimental limitations, the temporal dynamics of these cells are not often fully exploited to get biological insights \cite{klein2025mapping}. Furthermore, knowledge about cell types and their interactions, such as how they naturally coexist and evolve in healthy or diseased states, remains a challenge for the omics data science community \cite{lahnemann2020eleven,palla2022spatial,tanay2017scaling,cang2020inferring}. From a computational perspective, a central challenge is to develop methods that accurately represent, compare, and integrate temporal and spatial dependencies across different biological scales, conditions, and sample types \cite{velten2023principles}.

Spatial transcriptomic datasets usually arise from high-throughput omics technologies that measure molecular profiles across space or time, providing a detailed snapshot of the molecular state of the biological system \cite{velten2023principles}. Such measurements across space or over time, whether within a single omics layer or across multiple layers, reveal changes in molecular states and biological features that are invisible in static measurements. Spatial profiling has been particularly useful for reconstructing how tissues are organized during development \cite{larsson2021spatially,lewis2021spatial,palla2022spatial,seferbekova2023spatial}, for informing cost-effectiveness analyses in disease elimination strategies \cite{antony2024reducing,smith2026health,oliver2025uncertainty}, and for capturing variation within structures such as tumors, which can improve patient stratification \cite{rao2021exploring,mathews18}. In a similar way, time-series measurements, which are usually not independently and identically distributed, explain how molecular processes change over time, such as during development or in response to perturbations. Methods involving time-series data have included using optimal transport in conjunction with bandit approaches, as studied in \cite{oliver2026learning}. A benefit to approaching biological data from this perspective is the identification of temporal gene expression patterns linked to disease \cite{bar2012studying,cardoso2020developmental}.

However, approaches that involve only spatial or only temporal components of biological data have important limitations. For example, observations are often spatially or temporally correlated, as nearby samples in space or time tend to exhibit similar molecular behavior due to underlying biological relationships. In addition, spatial locations and time points usually need to be aligned or normalized before comparisons can be made between individuals or biological scales. To address these challenges, several methods have been designed to incorporate spatial and temporal structure in biological analyses. Some of the popular methods in the literature include regression-based modeling; parametric-based methods such as logistic growth models \cite{creswell2020high} and ODEs and PDEs \cite{karlebach2008modelling,maini2012turing}; and non-parametric-based methods such as splines \cite{van2020trajectory,Banerjee2024trajectory} and Gaussian processes \cite{Williams_1996}. To fully exploit the geometric nature of the underlying datasets, alternatives to functional dependent methods include graph-based modeling \cite{zhu2018identification,poury19}, Markov random fields (MRFs) \cite{elman1990finding}, Laplacian-based methods \cite{oliver2025laplace,calder2020poisson,shi2017weighted,nadler2009semi}, and graph neural networks \cite{zhou2020graph,hu2021spagcn}. Other existing methods outside of model-based approaches include the use of summary statistics to account for temporal and spatial proximity,  such as temporal or spatial autocorrelations using aggregated summary statistics to infer trajectory \cite{ghazanfar2020investigating}. Ultimately, the choice of the model depends on the following factors: computational scalability, model interpretability, and the biological question of interest.

In this paper, we adopt a geometric perspective based on
optimal transport (OT) theory \cite{Villani_Old} to understand the underlying dynamics of spatial-temporal models. Optimal transport provides a natural mathematical framework for comparing complex biological systems while respecting their intrinsic geometry. For a comprehensive overview of how optimal transport methods have been implemented in biological applications, see \cite{bunne2024optimal} for more details. 
Our two main tools are the Gromov-Wasserstein (GW) distance \cite{gwdist}, which compares objects that lie in distinct metric measures spaces, and Ollivier-Ricci (OR) curvature \cite{Ricci_curvature_Ollivier2009}, a discrete formulation of classical Ricci curvature for graphs and metric spaces.

Using GW distances, one can compute distances between graphs that have different numbers of nodes, making GW particularly well suited for comparing gene co-expression networks across developmental stages where the sizes of the network may vary.
The application of the GW distance in biological contexts also has precedent; for example, \cite{biologygw} implemented a GW-like distance, called the $\GW_\tau$ distance, between biological time series, which fixes the distance functions of each trajectory relative to the starting point of each trajectory, reducing the computational complexity of the GW optimization problem. This work further incorporated the use of fused Gromov-Wasserstein barycenters as a tool to average replicate trajectories, with the barycenter functioning as a single representative curve. In cell biology, \cite{doi:10.1089/cmb.2021.0446} proposed a method called single cell alignment with optimal transport (SCOT), a fully unsupervised algorithm that utilizes GW optimal transport to align single-cell multi-omics data by viewing cells as static point clouds and then aligning across different sequencing modalities.  However, despite its geometric versatility, the main disadvantage of the GW distance is that it is very computationally expensive, since the optimization problem is non-convex and in general is an NP-hard quadratic programming problem \cite{gwdist}.

In the context of single-cell dynamics, Ricci curvature appears through its synthetic formulation in the Wasserstein space via the Lott–Sturm–Villani theory, where curvature controls how probability measures evolve along geodesics in the space of cell-state distributions \cite{lott2009ricci}. 
Our interest in the combination of Ricci curvature with biological network data is motivated by the observation that curvature provides an interpretable geometric measure of network robustness and organization.
In particular, \cite{Graph_Curvature_Sandhu2015} combined OR curvature with gene expression networks, viewed as graphs, to identify significant differences in curvature between healthy and cancerous cellular networks. A key insight from \cite{Graph_Curvature_Sandhu2015} was that cancerous cellular networks exhibited statistically significant higher curvature than normal, healthy cellular networks. Other work in \cite{baptista2024} studied Ricci curvature and Ricci flows in the context of biological network rewiring and cellular differentiation. Notably, this work incorporated a temporal component: Ricci flow trajectories between initial and terminal stages of cellular differentiation gene expression datasets obtained accurate temporal ordering for intermediate time steps. These results align well with two of our contributions, which involve using GW geodesics and prototype curvature distributions at temporal stages as methods to order and classify gene expression networks at different spatial locations.

An extension of OR curvature known as \textit{dynamical Ollivier-Ricci curvature} was introduced by \cite{Unfolding_the_m_Gosztolai2021}, which measures how much two diffusion processes starting at neighboring nodes overlap after spreading for a pseudotime $\tau$, allowing edge curvatures to vary with $\tau$ rather than remaining in a fixed embedding space.

Intuitively, positive Ricci curvature corresponds to the contraction of mass along optimal transport paths, meaning that populations of cells (viewed as distributions over gene expression states) tend to mix and concentrate rather than disperse \cite{huynh2024topological}. In single-cell RNA-seq modeling \cite{baptista2024}, this translates into stability of developmental trajectories and robustness of differentiation pathways under perturbations. Conversely, non-positive curvature allows branching and divergence of trajectories, reflecting heterogeneous fate decisions. This viewpoint becomes particularly powerful when interpreting cell dynamics as gradient flows and PDEs of free energy functionals (such as JKO flows) ~\cite{bunne2022proximal,tong2020trajectorynet}.

Modern formulations of data alignment problems that arise in single-cell dynamics often seek to jointly infer correspondences between multiple datasets or modalities while respecting both local geometry and global consistency constraints. 
Single-cell data integration increasingly requires the simultaneous alignment of multiple datasets, developmental stages, or modalities while preserving both local geometry and global biological structure.
This naturally leads to extensions of GW distance beyond pairwise matching, such as co-optimal transport, which aligns both features and samples in a joint framework \cite{Redko_2020}. In this setting, one simultaneously optimizes a family of transport plans that are coupled through shared latent structure, rather than solving independent pairwise problems. This perspective is particularly powerful when data is represented as multi-level hypergraphs, where nodes encode cell states and higher-order hyperedges encode interactions, lineage relations, or shared regulatory structure \cite{chowdhury2024hypergraph}. 

Together, our work presents a unified geometric framework for analyzing spatial transcriptomic data across multiple scales. By combining graph curvature, Gromov--Wasserstein geometry, and co-optimal transport, we provide complementary tools for developmental stage characterization, trajectory analysis, and higher-order network comparison.

\subsection{Contributions}

The main contributions of this work are as follows:
\begin{enumerate}
    \item We develop a curvature-based framework for network-structured data with associated spatial and temporal information, showing that edgewise OR curvature distributions can provide discriminative signatures for comparing and classifying networks across temporal stages. We further evaluate the framework using \textit{Drosophila} (also known as the ``fruit fly'') spatial transcriptomics dataset \cite{wang2025drosophila} and quantify its   effectiveness in analyzing developmental changes 
     through OR curvatures of stage-wise gene co-expression networks. These results  provide additional empirical evidence, alongside \cite{Graph_Curvature_Sandhu2015} and \cite{baptista2024}, that OR curvature can serve as  an effective feature for differentiating biological states.

    \item We propose a trajectory analysis of biological network data based on Gromov-Wasserstein (GW) geodesics. While prior work \cite{baptista2024} incorporated Ricci flow trajectories to study cellular differentiation between initial and terminal stages of development, to the best of our knowledge this is the first application of GW geodesics to biological developmental trajectories. Our analysis shows that GW geodesic interpolation describes developmental progression between initial and terminal stages, with the curvature distributions of intermediate geodesics closely matching those of observed intermediate stages. Furthermore, curvature distributions along the GW trajectory accurately recover the temporal ordering of real data, demonstrating that GW geodesics provide an effective framework for reconstructing temporal changes in gene expression networks, even when intermediate observations are unavailable. These results correspond with \cite{baptista2024} using Ricci flow trajectories. Together, these findings further support the conclusion that Ollivier-Ricci curvature distributions encode biologically meaningful information about the evolving structure of gene expression networks over time.

    \item We present a hypernetwork representation of the developmental stages of gene expression network data to incorporate multilayer spatial and temporal information combined with co-optimal transport (COOT). This analysis reveals that higher-order representations reveal both the spatial and temporal organization of \textit{Drosophila} development. By quantifying similarities between developmental stages using co-optimal transport, we recover the expected sequence of embryonic and early larval development as described in \cite{wang2025drosophila}. 
    These findings demonstrate that hypernetwork representations preserve biologically meaningful developmental structure across multiple scales and, combined with COOT, provide a framework for characterizing developmental trajectories and 
    major developmental transitions. 
\end{enumerate}

\subsection{Outline}

Background and preliminary material on GW geodesics, co-optimal transport distance, Ricci curvature, and related constructions are provided in \Cref{sec:background}. The main methodology is presented in \Cref{sec:method}.  In particular, the curvature-based classification of developmental stages is described in \Cref{subsubsec:realcurvaturecomp}; GW-geodesic-based curvature trend analysis and temporal rank ordering  are presented in \Cref{subsection: trend analysis with OR curvature,sec:geodesicclassification}; the dynamic OR curvature analysis is presented in \Cref{sec:dynamic-curvature-distributions}; and the hypernetwork-based co-optimal transport analysis is described in \Cref{sec:coot_development_trajectory}. The corresponding applications and results for the \textit{Drosophila} dataset are presented in \Cref{sec:application}. A summary of the results and a discussion of future work are provided in \Cref{sec:discussion}.

\section{Background}\label{sec:background}

Gene co-expression networks~\cite{weirauch2011gene} encode the strength of pairwise associations among genes across tissue samples collected at different developmental stages or spatial locations. In this work, we first represent such spatiotemporal gene-expression data as finite graphs and, more generally, as finite metric measure spaces obtained by equipping these graphs with a distance and a node measure. To incorporate spatial resolution more explicitly, we also use hypernetworks ~\cite{chowdhury2024hypergraph} to encode higher-order relationships among genes or spatial slices. 

At the graph level, one common construction is to measure the association between each pair of genes using the correlation of their measured expression across samples (see ~\cite{horvath2008geometric} for more details). A large correlation, in magnitude or after a chosen similarity transformation, indicates a strong pairwise relationship between two genes. Genes are treated as nodes, and an edge is placed between two genes when their association exceeds a prescribed threshold. Conversely, pairs with weak association are not connected. Since correlations may be negative, one may either retain signed edge weights or transform the correlations into nonnegative similarity scores following \cite{Graph_Curvature_Sandhu2015}, depending on the modeling goal. After thresholding, the resulting adjacency matrix defines a graph representation of the gene-expression network, in which nodes correspond to genes and edges encode sufficiently strong pairwise co-expression relationships.

\subsection{Gromov-Wasserstein distance and geodesics}
 We recall that a Polish metric space $(X,d)$ is a complete, separable metric space. The set of Borel probability measures on $X$ is denoted as $\cP(X)$.  A metric measure space $\mathcal{X}:=(X, d_X, \mu_X)$ is a compact metric space $(X, d_X)$ endowed with a probability measure $\mu_X$. We first recap some preliminary notions of distance between probability measures in the same and distinct metric measure spaces, as summarized below.

\begin{Def} [$p$-Wasserstein distance, \cite{Villani_Old}] \label{definition: Wasserstein Distance Abstract}
	Let \((X, d)\) be a Polish metric space, and let \(p \in [1, \infty)\). For any two probability measures \(\mu, \nu\) in \(\mathcal{P}(X)\), the Wasserstein distance of order \(p\) between  \(\mu\) and \(\nu\) is defined by
	\[
		W_{p}(\mu,\nu) = \left( \inf_{\pi \in \Cpl(\mu,\nu) } \int_{X \times X} d(x,y)^{p}\textrm{d}\pi(x,y) \right)^{1/p},
	\]
     where $\Cpl(\mu, \nu)$ is the set of coupling, i.e.,  probability measures on $X \times X$ with marginals $\mu$ and $\nu$ and $p \in [1, \infty)$.
\end{Def}

\begin{remark}[Coupling and Transport Map, \cite{Villani_Old}]
A coupling can be defined generally for probability measures $\mu_X \in \mathcal{P}(X)$ and $\mu_Y \in \mathcal{P}(Y)$ on different spaces $X$ and $Y$. To standardize notation, we introduce it using the pushforward operation. 

\begin{equation*}
    \Cpl(\mu_X, \mu_Y):= \{\pi \in \cP(X \times Y): (\Proj_X)_\sharp \pi = \mu_X, (\Proj_Y)_\sharp \pi = \mu_Y \},
\end{equation*}
where $\Proj_X:X\times Y \to X$ and $\Proj_Y: X\times Y \to Y$ are the canonical projections. On the other hand, the pushforward of the measure $\mu_X$ by a Borel measurable map $T: X \rightarrow Y$, denoted by  $T_\sharp \mu_X$,  satisfies
\begin{equation*}
    (T_\sharp \mu_X) (B) = \mu_Y (B),
\end{equation*}
for any measurable set $B\subset Y$.
\end{remark}

The $p$-Gromov-Wasserstein (GW) distance was first introduced by M\'emoli in \cite{gwdist} as a method to compare metric measure spaces (mm-spaces). Later work by Sturm \cite{Strum_2020} developed many theoretical aspects of a distance between metric measure spaces related to the GW distance, such as geodesics and curvature bounds, with particular focus on the case $p=2$. We formally define the $p$-th GW distance as follows.

\begin{Def}[$p$-Gromov-Wasserstein distance, \cite{gwdist}]\label{gwdistdef}
    Let $\mathcal{X}=(X, d_X, \mu_X)$ and $\mathcal{Y}=(Y, d_Y, \mu_Y)$ represent mm--spaces where $d_X$ and $d_Y$ are metrics and $\mu_X$ and $\mu_Y$ are probability measures. Then
        $$\GW_{p}(\mathcal{X}, \mathcal{Y}) =
        \frac{1}{2}\inf_{\pi \in \Cpl(\mu_X, \mu_Y)} \left( \int_{(X \times Y)^2} |d_X(x,x') - d_Y(y,y')|^p \textrm{d} \pi(x,y) \textrm{d}\pi(x', y') \right)^{1/p}.$$
\end{Def}

The GW distance defines a metric on the space of isomorphism classes of mm--spaces, where mm--spaces are defined as \textit{isomorphic} if $\GW_p(\mathcal{X}, \mathcal{Y}) = 0$ \cite{gwdist}. Let $\lb \mathcal{X}\rb$ denote an isomorphism class of mm--spaces. We are interested in the case of $p = 2$ as this is the case implemented numerically; see \cite{flamary2021pot}. 
The space of isomorphism classes of mm--spaces equipped with the $p$-th GW distance has geodesic structure. In \cite{Strum_2020}, the form of geodesics in GW space was completely characterized. Sturm's geodesic construction ensures that, by definition, every geodesic of this form is a constant speed geodesic. Geodesics between two isomorphism classes $\lb \mathcal{X} \rb$ and $\rb \mathcal{Y} \lb$ take the explicit form shown below. 

\begin{Thm}[{\cite[Theorem~3.1]{Strum_2020}}]
    Let $\lb\mathcal{X}_0 \rb$ and $\lb \mathcal{X}_1 \rb$ be equivalence classes of mm--spaces, and let
$\overline m\in\Cpl(m_0,m_1)$ be an optimal coupling. Define, for $t\in[0,1]$,
\begin{equation}\label{eq:GW geodesic}
    \lb \mathcal{X}_t \rb := \lb X_0\times X_1,\ d_t,\ \overline m \rb,
\qquad
 d_t\big((x_0,x_1),(y_0,y_1)\big):=(1-t)d_0(x_0,y_0)+t\,d_1(x_1,y_1).
\end{equation}

Then $\lb (\mathcal{X}_t)_{t\in[0,1]} \rb$ is a (constant-speed) geodesic in the GW space connecting $\lb \mathcal{X}_0 \rb$ and $\lb \mathcal{X}_1 \rb$.
\end{Thm}

 Numerically, however, GW geodesics are difficult to compute since they live in the product space $X_0 \times X_1$. For this reason, GW geodesics are closely connected to the concept of the \textit{GW barycenter}, introduced in \cite{peyre2016gromov}. The GW barycenter arises from the Fr\'echet variational principle applied to the space of isomorphism classes of metric measure spaces equipped with the GW metric, and it offers a practical method to define and compute geodesics in Gromov-Wasserstein space. Formally, we define a GW barycenter in the finite case as follows.

\begin{Def}[Finite GW barycenter, \cite{peyre2016gromov}] \label{definition: GW barycenter}
Given finite metric measure spaces $\mathcal X_s=(X_s,d_s,\mu_s)$, $s=1,\ldots,S$, with
$X_s=\{x_1^s,\ldots,x_{N_s}^s\}$, where $q_s$ is the probability vector defining discrete measure $\mu_s=\sum_{i=1}^{N_s}(q_s)_i\delta_{x_i^s}$, and  $D_s\in\mathbb R^{N_s\times N_s}$ denotes the distance matrix $(D_s)_{ij}=d_s(x_i^s,x_j^s)$. Given chosen weights $\lambda:= (\lambda_s)_{s=1}^S$, and a target metric space size $N$ and a fixed target probability vector $q$,  the GW barycenter of $(D_s)^S_{s=1}$ is given by, 
        $$\arg\min_{D \in \mathbb{R}^{N \times N}} \sum_{s=1}^S \lambda_s \GW_p^p(\mathcal X_s, \mathcal X),$$   
where $\mathcal X$ is a mm-space with distance matrix $D$ and associated probability vector $q$. Typically $p=2$ is chosen. 
 \end{Def}
If no prior information is given, the probability vectors are usually set to have uniform weights. 
In \cite{peyre2016gromov}, the matrices $D_s$ can be chosen as more general similarity matrices that do not need to satisfy positivity or the triangle inequality.

GW geodesics and GW barycenters are connected through the two-space barycenter problem. Specifically, specializing the GW barycenter formulation to $S=2$ with weights $(1-t,t)$, $t\in [0,1]$, yields the free-support barycenter problem
\begin{equation}\label{eq:free support barycenter}
    \mathcal{X}_B^t
    \in
    \operatorname*{arg\,min}_{\mathcal X}
    \left\{
    (1-t)\GW_2^2(\mathcal{X}_0,\mathcal X)
    +
    t\GW_2^2(\mathcal{X}_1,\mathcal X)
    \right\}.
\end{equation}
By \cite[Theorem~5.1]{beier2023multi}, any solution $\mathcal{X}_B^t$ is, up to mm-isomorphism, a time-$t$ point on a GW geodesic between $\mathcal{X}_0$ and $\mathcal{X}_1$ of the form given in \eqref{eq:GW geodesic}. In our numerical implementation, we also take $S=2$ and $p=2$, but restrict the candidate barycenter to finite mm-spaces with prescribed size $N$ and uniform probability vectors, as in \Cref{definition: GW barycenter}. Solving this constrained barycenter problem yields an $N\times N$ barycenter distance matrix $D_B^t$, which provides a fixed-size approximation of the corresponding free-support geodesic point.

\subsection{Measure hypernetworks}\label{sec:measure_hypernetworks}

Given the spatial-temporal structures in the dataset due to the availability of spatial coordinates for the gene expression measurements across different developmental stages, we introduce the notion of measure hypernetworks as well as a related distance. This framework enables comparisons of these multiple layers of structures. 

A measure network \cite{Chowdhury_2019} is a generalization of graphs and metric measure spaces, which is a tuple $\mathcal{N}_X =(X,  \omega_X, \mu_X)$ where $X$ is a Polish space and $\mu_X$ is a Borel probability measure, while $\omega_X:X\times X \rightarrow \mathbb{R}$, often referred to as the network function, is a bounded, measurable function. Classical finite graphs with weighted edges are
special cases of measure networks, obtained by taking $X$ as the vertex set, $\mu_X$ as the uniform probability measure on vertices, and $\omega_X$ as an adjacency or weight function. The GW distance (\Cref{gwdistdef}) can also be generalized to the measure network setting by replacing the distance functions by the network functions. The resulting network based GW distance was defined in \cite[Section 2.4]{Chowdhury_2019}.

Classical hypergraphs \cite{zhou2006learning} generalize graphs by allowing edges to connect more than two nodes. To accommodate this higher-order relational structure in a measure-theoretic framework, we define \textit{measure hypernetworks}. More precisely, the GW framework provides a way to compare square-shaped kernels $\omega_X : X \times X \to \mathbb{R}$, where a single space indexes both arguments (e.g.\ samples compared to samples). In contrast, many datasets naturally give rise to rectangular kernels $\hyperomegaX : X \times X' \to \mathbb{R}$, where $X$ and $X'$ represent different types of entities, typically samples and features. Applications may require aligning both samples and features simultaneously, such as spatial transcriptomics datasets to understand the interaction between the temporal and spatial profiles (samples) of the corresponding genes (features). This motivates the notion of a measure hypernetwork of~\cite{Chowdhury_2019} which we formally define below.

\begin{Def}[Measure Hypernetwork {\cite[Definition 4]{chowdhury2024hypergraph}}]
A measure hypernetwork is a quintuple
\[\mathcal{H} = (X,\mu_X, X',\mu_{X'}, \omega_{X, X'}),\]
where $(X,\mu_X)$ and $(X',\mu_{X'})$ are Polish spaces with full-support Borel probability measures, and $\hyperomegaX : X \times X' \to [0,\infty)$ is a bounded, measurable function, called the hypernetwork function.
\end{Def}

\begin{remark}
  Classical finite hypergraphs can be recovered by taking $X$ finite and $X'$ being a collection of subsets of $X$, $\mu_X$ and $\mu_{X'}$ as uniform distributions, and $\omega_{X, X'}$ as the incidence function.  
\end{remark}

\begin{example}[Measure hypernetwork for spatial transcriptomic data] Given spatial transcriptomic data \cite{wang2025drosophila} at a fixed developmental stage, let $X$ denote a set of spatial locations (or cells) and let $X'$ denote a set of genes. A dataset can then be represented as a measure hypernetwork
$\mathcal H_X=(X,\mu_X,X',\mu_{X'},\omega_X),
$
where the hypernetwork function
$\hyperomegaX:X\times X' \to \mathbb R
$
encodes gene expression levels, with $\hyperomegaX(x,x')$ representing the expression of gene $x'$ at location $x$. We note that $\mu_X$ is a probability measure on spatial locations (e.g., the uniform measure assigning equal weight to each location) and $\mu_{X'}$ is a probability measure on genes (e.g., the uniform measure assigning equal weight to each gene).  A second hypernetwork $\mathcal H_Y$ may represent spatial transcriptomic data at a different developmental stage. 

\end{example}   

The above hypernetwork framework provides a natural framework for analyzing sample--feature data, such as spatial transcriptomics datasets, while preserving the full relational structure encoded by the expression matrices.  We now introduce a distance between measure hypernetworks, which naturally extends the co-optimal transport framework of \cite{Redko_2020}. In that work, the focus was restricted to measure hypernetworks with finite underlying sets, primarily motivated by applications to data matrices. The formulation below generalizes this framework to include infinite measure spaces (see also examples in \cite{chowdhury2024hypergraph}). This extension is essential for obtaining a complete metric space, as established in \cite[Theorem 1]{chowdhury2024hypergraph}. In practice, however, computations are always performed on finite approximations, reducing back to the original finite setting.
\begin{Def}[Hypernetwork $p$-Distance {\cite[Definition 5]{chowdhury2024hypergraph}}]\label{Def:COOT}
Let 
\[\mathcal{H}_X= (X,\mu_X,X',\mu_{X'},\omega_{X, X'}), \quad \mathcal{H}_Y = (Y,\mu_Y,Y',\mu_{Y'},\omega_{Y,Y'})\] 
 be measure hypernetworks for $p \in [1,\infty)$. For couplings $\pi \in \Cpl(\mu_X,\mu_{Y})$ and $\pi' \in \Cpl(\mu_{X'},\mu_{Y'})$, define the $p$-th co-optimal distortion 
\[
\operatorname{dis}_p(\pi,\pi') = 
\left(
\int_{X \times Y} \int_{X' \times Y'} 
|\hyperomegaX(x,x') - \hyperomegaY(y,y')|^p \, \textnormal{d} \pi(x,y)\, \textnormal{d} \pi'(x',y')
\right)^{1/p}.
\]
The $p$-hypernetwork co-optimal transport distance is defined as 
\begin{equation}\label{eq:COOT}
    \textnormal{COOT}_p(\mathcal{H}_X,\mathcal{H}_Y) = \frac{1}{2}
\inf_{\pi \in \Cpl(\mu_X,\mu_{Y})} \inf_{\pi' \in \Cpl(\mu_{X'},\mu_{Y'})} \operatorname{dis}_p(\pi,\pi').
\end{equation}
\end{Def}

This distance generalizes the GW distance to higher-order structures by simultaneously matching nodes and hyperedges (edges connecting more than two nodes) and preserving the full relational information in $\hyperomegaX,\hyperomegaY$. In particular, it provides simultaneous alignments of spatial locations and genes across different developmental stages for the fruit fly spatial transcriptomic data. We note that a major limitation of the COOT distance is that it requires measures of equal mass and is highly sensitive to outliers although recent extensions to the conic Co-Optimal transport distance have addressed this limitation (see \cite{oliver2025conic}).

\subsection{Ricci Curvature}\label{subsection: Ricci Curvature}
Ollivier–Ricci (OR) curvature \cite{Ricci_curvature_Ollivier2009, Graph_Curvature_Sandhu2015} on a graph intuitively measures the deviation of the graph from being locally grid-like analogously to being `flat’ in continuous spaces. ‘Flatness’ of a network can be understood in terms of its local connectivity: the distance of a pair of nodes is the same as the average distance of their neighborhoods. Thus positive (or negative) OR curvature of an edge indicates that it resides in a region of the graph that is more (or less) connected than a grid. 
We present the standard definition for OR curvature on a graph.

\begin{Def}[Ollivier-Ricci Curvature on a Graph \cite{Ricci_curvature_Ollivier2009, Graph_Curvature_Sandhu2015}] \label{definition: OR curvature}
	Let \(G = (V,E, w) \) be an undirected weighted graph. For $x\in V$, define  probability measure \(\mu_{x}\) as  
\begin{equation}\label{eq:nodal measure}
 \mu_{x}(z) := \frac{w(x,z)}{\sum_{\{x,v\}\in E} w(x,v)} \text{ if } \{x,z\}\in E,  \text{ otherwise } \mu_x(z) = 0,   
\end{equation}
where \(w(x,y)\) denotes the weight of edge $\{x,y\}$. Then the \textit{Ollivier-Ricci curvature} along edge $\{x,y\}\in E$ is
	\[
		\kappa_{x,y} \coloneqq 1 - \frac{W_{1}(\mu_{x}, \mu_{y})}{d(x,y)},
	\]
    where  $d$ denotes the corresponding shortest-path distance on the graph, and $W_1$ is the 1-Wasserstein distance defined in \eqref{definition: Wasserstein Distance Abstract} with $p=1$.
\end{Def}

The nodal probability measure \(\mu_{x}\) can be regarded as the distribution of a one-step random walk starting from \(x\) with interaction strength proportional to the weights. This notion can be extended from one step random walks to Markov diffusion processes \cite{Unfolding_the_m_Gosztolai2021}, yielding \textit{dynamical Ollivier-Ricci} curvature on a graph. A continuous time diffusion on a graph \(G\) is constructed by the standard procedure \cite{Spectral_Graph_Chung1997} of defining the normalized graph Laplacian matrix \(\mathbf{L} \coloneqq K^{-1}(K - A)\), where \(A\) denotes the weighted adjacency matrix and \(K\) is the diagonal matrix of node degrees with \(K_{ii} = \sum_{j} A_{ij}\). Then the probability measure of the diffusion starting from the unit mass \(\delta_{x}\) on node \(x\) evolves according to
\begin{equation} \label{eqn: Markov Diffusion probability distribution}
	\theta_{x}(\tau) = \delta_{x}e^{-\tau\mathbf{L}}.
\end{equation}
With this, the dynamical Ollivier-Ricci curvature can be defined analogously to classical OR curvature.

\begin{Def}[Dynamical Ollivier-Ricci Curvature, Equation (2) \cite{Unfolding_the_m_Gosztolai2021}]\label{def: dynamic ORcurv}
	Given an undirected weighted graph \(G = (V,E, w)\) with induced geodesic distance \(d\). Consider the probability measures \(\theta_{x}\) defined by \cref{eqn: Markov Diffusion probability distribution}. Then the \textit{dynamical Ollivier-Ricci curvature} along edge $\{x,y\}\in E$ is
	\[
	\kappa_{x,y}(\tau) \coloneqq 1 - \frac{W_{1}(\theta_{x}(\tau), \theta_{y}(\tau))}{d(x,y)}, \quad \tau > 0
	\]
\end{Def}

\section{Methodology and analysis}\label{sec:method}

\subsection{Curvature as geometric signatures}\label{sec: curvature signatures}

For each developmental stage, we represent the gene expression profile as an undirected weighted gene co-expression network $G=(V,E,w)$, whose nodes are the selected genes and whose weighted edges encode the retained nonnegative co-expression strengths derived from the processed gene--gene correlations described in Section \ref{sec:data preprocessing}.

We compute OR curvature in the following two settings.
\begin{enumerate}
     \item \textbf{Graph setting.}
     For an observed gene co-expression network $G=(V,E,w)$, we compute edge-level OR curvature on the retained edges following \Cref{definition: OR curvature}, with respect to the graph metric $d_G$. We denote the resulting edge-indexed curvature distribution by
\[
    \boldsymbol{\kappa}_G \coloneqq \bigl(\kappa_{x,y}\bigr)_{\{x,y\}\in E},
\]
where $\kappa_{x,y}$ denotes the curvature of the edge connecting the co-expressed genes $x$ and $y$.

\item \textbf{Metric measure (mm) space setting.}
To accommodate the trajectory analysis via the GW geodesic in \eqref{eq:GW geodesic}, we view each observed graph as a finite metric measure space $\mathcal{X}=(X,d_X,\mu_X)$, where $X=V$ is the set of nodes, $d_X$ is the graph metric based on shortest-path distances, and  $\mu_X$ is the uniform probability measure on $X$. Since the interpolated spaces are finite mm-spaces rather than observed gene co-expression networks, their local measures are constructed from the metric structure instead of edge weights. 
Writing the interpolated spaces as $\mathcal{X}_t=(X_t,d_t,\mu_t)$, $t\in[0,1]$, we define the associated curvature distribution in the following way.

We construct a graph $G_{B}^{w,t} = (X_t, E_t, w_t)$ whose induced graph metric approximate the mm-space metric $d_t$, using the graph construction described below in \Cref{subsection: graph approximations}.  Here the node set $X_t$ is chosen as the set of common genes as described in \Cref{sec:data preprocessing}, and the edge set $E_t$ is determined from the GW barycenter associated with the underlying unweighted graphs, whereas the edge-weight function $w_t$ is obtained from the corresponding GW barycenter associated with the weighted graphs. Then, the graph-based OR curvature distribution $\boldsymbol{\kappa}_{G_B^{w,t}}$ is computed. 

\end{enumerate}
\subsubsection{Graph approximations of interpolated mm-spaces} \label{subsection: graph approximations}
We describe how the GW barycenter in \Cref{definition: GW barycenter} is used to construct weighted graphs that approximate the interpolated mm-spaces, outlined in \Cref{sec: curvature signatures}.

The construction proceeds in two steps to obtain a weighted graph approximation from the interpolated metric information. First, the graph topology is determined from the GW barycenter of the unweighted shortest-path distance matrices; in particular, two vertices are connected when their barycenter distance is sufficiently close to one. Second, the GW barycenter of the weighted shortest-path distance matrices is used to assign weights to the resulting edges.

Consider two weighted graphs $G_0$ and $G_1$ of $N$ vertices as finite mm-spaces with distance matrices $D_0, D_1 \in \mathbb{R}^{N\times N}$ respectively, and uniform probability measure.  Let \(D_B^t\) denote the GW barycenter matrix corresponding to the GW barycenter as in \cref{definition: GW barycenter}; in our case, that is
\begin{equation} \label{eq: distance matrix interpolation}
	D_{B}^{t} = \arg\min_{D \in \mathbb{R}^{N \times N}} t \GW_2^2(D_0,D) + (1-t)\GW_2^2(D_1,D).
\end{equation}

\paragraph{Step 1: construction of the graph topology.}
Let $G_{0}$ and $G_{1}$ be connected weighted graphs with $N$ vertices, and let $D(G_{0})$ and $D(G_{1})$ denote their unweighted shortest-path distance matrices. Let $D_B^{t}$ be their GW barycenter defined by \cref{eq: distance matrix interpolation}. The associated unweighted graph is defined by
\begin{equation}\label{eqn: distance matrix to unweighted graph}
G_{B}^{t} = (V,E^{t}), \qquad
V = \left\{1,\ldots,N\right\}, \qquad
E^{t}
=
\left\{
\{i,j\}
\;\middle|\;
i,j\in V,\;
\left|D_{B}^{t}(i,j)-1\right|<\delta
\right\}.
\end{equation}
where $\delta\in(0,1)$ is a prescribed tolerance.
\begin{remark}
At the endpoints, $G_{B}^{0}=G_{0}$ and $G_{B}^{1}=G_{1}$ because $D_{B}^{0}=D_{0}$ and $D_{B}^{1}=D_{1}$. For intermediate values of $t$, the barycenter matrix $D_B^t$ need not be the shortest-path distance matrix of a graph. Consequently,
$D(G_B^t)\neq D_B^t$
in general, and $G_B^t$ can be  interpreted as a graph approximation of the interpolated metric structure.
\end{remark}

\paragraph{Step 2: construction of the weighted graph.}
Let $D^{w}(G_{0})$ and $D^{w}(G_{1})$ denote the weighted shortest-path distance matrices of $G_{0}$ and $G_{1}$, respectively, and let $D_{B}^{w,t}$ denote the corresponding GW barycenter distance matrix defined by \cref{eq: distance matrix interpolation}. The weighted graph is then defined by
\begin{equation} \label{eqn: distance matrix to weighted graph}
G_{B}^{w,t} = ( G_{B}^{t}, w), \qquad w(\left\{ i,j \right\}) = D_{B}^{w,t}(i,j),
\end{equation}
where $G_{B}^{t}$ is constructed in Step 1. 

\begin{remark}
For disconnected graphs, infinite entries in the unweighted shortest-path distance matrices, which determine the graph topology as in Step 1, are replaced by a large finite value $\gamma\gg0$. In the weighted matrices, which encode interaction strengths, missing interactions are assigned value zero. The two-step construction is then applied to the resulting finite matrices:
\begin{equation}\label{eqn:disconnnected distance thresholding}
\hat{D}(G)(i,j)
\coloneqq
\begin{cases}
D(G)(i,j), & \text{if } D(G)(i,j)<+\infty,,\\
\gamma, & \text{otherwise},
\end{cases}
\qquad
\hat{D}^{w}(G)(i,j)
\coloneqq
\begin{cases}
D^{w}(G)(i,j), & \text{if } D(G)(i,j)<+\infty,\\
0, & \text{otherwise}.
\end{cases}
\end{equation}

\end{remark}

\subsection{Inference with curvature and Gromov-Wasserstein geodesics}
\subsubsection{Curvature  comparison between real and interpolated data}\label{sec:geodesicclassification}
To assess the curvature distributions measured along the GW geodesics in comparison to the curvature distributions in real data, we propose a classification framework that predicts the closest developmental stage. This approach provides a more objective and robust evaluation of whether GW geodesics can effectively represent the underlying real data, versus visually comparing trends in mean and median curvature. We use three distinct methods for this comparison, each using a discrepancy measure $\delta(\boldsymbol{\kappa}_G, \boldsymbol{\kappa}_{G_t})$ defined as follows.

\begin{enumerate}
    \item \textbf{Scalar Approach (Mean and Median).} The curvature distribution is reduced to a single summary statistic $\mathcal{S}(\boldsymbol{\kappa}) \in \{\bar{\kappa}, \tilde{\kappa}\}$, where
        \[
    \bar{\kappa} = \frac{1}{|E|}\sum_{\{x,y\}\in E}\kappa_{x,y}, \qquad
    \tilde{\kappa} = \operatorname{median}\bigl\{\kappa_{x,y}:\{x,y\}\in E\bigr\},
    \]
    and the discrepancy is
    \[
    \delta_\mathrm{scalar}(\boldsymbol{\kappa}_G, \boldsymbol{\kappa}_{G_t}) := \bigl|\mathcal{S}(\boldsymbol{\kappa}_G) - \mathcal{S}(\boldsymbol{\kappa}_{G_t}) \bigr|.
    \]
    
    \item \textbf{Wasserstein Distance.} Let $\mu_G$ and $\mu_{d_t}$ denote the empirical distributions of $\boldsymbol{\kappa}_G$ and $\boldsymbol{\kappa}_{G_t}$, respectively. The discrepancy $\delta_W$ is the $1$-Wasserstein distance as defined in \Cref{definition: Wasserstein Distance Abstract},
    where we minimize over $\Cpl(\mu_G, \mu_{d_t})$, the set of all couplings of $\mu_G$ and $\mu_{d_t}$. Unlike $\delta_\mathrm{scalar}$, $\delta_W$ computes distances between the curvature distributions and is sensitive to differences in location, spread, and shape. We compute this distance between the raw, unstandardized curvature distributions. 
    \item \textbf{Symmetrized Kullback-Leibler (KL) Divergence.} Let $\hat{f}_G$ and $\hat{f}_{d_t}$ denote kernel density estimates (KDEs) fitted to the standardized curvature vectors
    \[
    \tilde{\boldsymbol{\kappa}} := \frac{\boldsymbol{\kappa} - \bar{\kappa}}{\sigma_\kappa}, \qquad \sigma_\kappa = \Bigl(\tfrac{1}{|E|} \sum_{\{x,y\} \in E} (\kappa_{x,y} - \bar{\kappa})^2 \Bigr)^{1/2}.
    \]
    Standardizing to zero mean and unit variance before fitting the KDE focuses the divergence on shape differences, such as skewness, modality, and tail weight, rather than differences between scale. Following the formula shown in the first line of Equation (7) in \cite{symmetrickl}, we define our discrepancy measure in this setting as the symmetrized KL divergence: 
    \[
    \delta_\mathrm{KL}(\boldsymbol{\kappa}_G, \boldsymbol{\kappa}_{G_t}) := \frac{1}{2} \Bigl[
    D_\mathrm{KL}\bigl( \hat{f}_G \, \|\,\hat{f}_{d_t}\bigr) + D_\mathrm{KL}\bigl(\hat{f}_{d_t}\,\|\, \hat{f}_G \bigr)
    \Bigr],
    \]
    where $D_\mathrm{KL}(p\|q) = \int p(u) \log\tfrac{p(u)}{q(u)}\,\mathrm{d}u$ (see Definition 1.1, Equation (1.3), in \cite{YAO2025104635}). Symmetrization ensures that $\delta_\mathrm{KL}$ does not depend on the choice of reference distribution.
\end{enumerate}

We then use these discrepancy measures in the following rank ordering procedure. For each real graph at observed time $t_r$, we compute the discrepancy $d(t_r, t_g)$ between its curvature distribution and that of every geodesic interpolation point $t_g \in \{t_1, \dots, t_{40}\}$ along the discretized GW geodesic connecting the first and last real graphs. The predicted time point is the geodesic index minimizing this discrepancy: $\hat{t}_r = \arg \min_{t_g} d(t_r, t_g)$. Repeating this for all real graphs yields a predicted ordering, which is then compared against the true, known temporal ordering using the rank correlation measures Kendall's $\tau$ and Spearman's $\rho$. 
For the purposes of this classification, we are interested in a relative temporal ordering more so than exact matching, and Kendall and Spearman rank correlations assess the association between true and predicted order without requiring predictions to land on the exact correct index. 

For each discrepancy measure, the classification procedure produces a predicted geodesic index $\hat{t}_r$ for each of the five real graphs. We then treat the true temporal order $(1, 2, 3, 4, 5)$ and the predicted order $(\hat{r}_1, \hat{r}_2, \dots, \hat{r}_5)$, where $\hat{r}_i$ is the rank of $\hat{t}_i$ among the five predicted indices, as two ranking of the same five items, and we compute Kendall's $\tau$ and Spearman's $\rho$ between them. A value near $+1$ indicates that the discrepancy measure recovers the correct developmental ordering: real graphs that are temporally later are reliably matched to later points along the geodesic. A value near 0 indicates no systematic relationship between predicted and true order; that is, the method's discrepancy measure carries no temporal signal beyond random chance. Negative values would indicate an inversion of the expected ordering.

Given two rankings of a collection of graphs, consider all pairs of items. A pair is called concordant if the two rankings agree on their relative order and discordant if they disagree. Then Kendall's $\tau$ is computed via the $\tau_b$ statistic:
\[\tau_b = \frac{P - Q}{(P + Q+ U)\sqrt{P + Q + T}},\]
where $P$ denotes the number of concordant pairs; $Q$ is the number of discordant pairs; $T$ is the number of tied pairs only in the true temporal order; and $U$ is the number of tied pairs only in the predicted order. In particular, this version of Kendall's $\tau$ was first introduced in \cite{kendalltau} as a method to account for potential ties. Assuming no ties, Spearman's $\rho$ is computed by
\[
\rho = 1 - \frac{6 \sum^n_{i=1}d_i^2}{n(n^2-1)},
\]
where $d_i$ is the rank difference for item $i$ and $n$ denotes the number of graphs. In the event of ties, fractional ranks for tied values are used.

We summarize this rank ordering procedure in \Cref{alg:rankordering}.

\begin{algorithm}[H]
    \caption{Curvature-Based Temporal Classification}
    \label{alg:rankordering}
    \begin{algorithmic}[1]
        \Require{
        Real graphs $G_1, \dots, G_n$ at known time points $t_1 < \dots < t_n$; geodesic interpolants $\{d_\tau\}_{\tau \in T}$ over a grid $T = \{\tau_1, \dots, \tau_M\}$; discrepancy measure $\delta \in \{\delta_\mathrm{scalar}, \delta_W, \delta_\mathrm{KL}\}$
        }
        \Ensure{
        Predicted geodesic indices $\hat{\tau}_1, \dots, \hat{\tau}_n$ and rank correlations $\tau_K$, $\rho_S$
        }

        \medskip
        \State \textbf{// Precompute geodesic curvature distributions}
        \For{each $\tau_j \in T$}
            \State Compute $\boldsymbol{\kappa}_{d_{\tau_j}} := \left(\kappa_{d_{\tau_j}}(x,y) \right)_{\{x,y\} \in E}$; remove non-finite entries
        \EndFor

        \medskip
        \State \textbf{//Classify each real graph}
        \For{each real graph $G_i$, $i = 1, \dots n$}
            \State Compute $\boldsymbol{\kappa}_{G_i} := (\kappa_{x,y})_{\{x,y\} \in E}$
            \State Remove non-finite entries from $\boldsymbol{\kappa_{G_i}}$
            \For{each $\tau_j \in T$}
                \State Compute $\delta_{ij} \leftarrow \delta\left(\boldsymbol{\kappa}_{G_i}, \boldsymbol{\kappa}_{d_{\tau_j}} \right)$
            \EndFor
            \State $\hat{\tau}_i \leftarrow \tau_j$ where $j^* = \arg \min_j \delta_{ij}$
        \EndFor

        \medskip
        \State \textbf{// Evaluate ordering recovery via rank correlation}
        \State Let $\boldsymbol{r} = (1, 2, \dots, )$ be the true rank vector
        \State Let $\boldsymbol{\hat{r}} = (\hat{r}_1, \dots, \hat{r}_n)$ be the predicted rank vector, where $\hat{r}_i = j_i^*$
        \State Compute Kendall's $\tau_K$ and Spearman's $\rho_S$ between $\boldsymbol{r}$ and $\boldsymbol{\hat{r}}$
    \end{algorithmic}
\end{algorithm}

\subsubsection{Predictions using curvature of real data}\label{subsubsec:realcurvaturecomp}
Our second classification task differs from the geodesic-based approach described in \Cref{sec:geodesicclassification} in both its reference objects and its evaluation criterion. Rather than assigning each real graph to a point along a continuous geodesic interpolation, we now classify spatial slices (ten slices per developmental stage) to one of the five developmental stages by comparing to the reference curvature distributions $\{\boldsymbol{\kappa}_{G_i}\}_{i=1}^5$, where each $G_i$ corresponds to the real graph at stage $i$ obtained by taking the common spatial slices and the common genes across all developmental stages. The discrepancy measures $\delta_\mathrm{scalar}$, $\delta_W$, and $\delta_{KL}$ are identical to those described in \Cref{sec:geodesicclassification}. Let $S = \{s_1, \dots, s_5\}$ denote the five developmental stages in temporal order, with associated reference curvature vectors $\boldsymbol{\kappa}_{G_i} := (\kappa_{x,y})_{\{ x,y\}\in E_i}$ extracted from $G_i$. For each stage $s_i$ and spatial slices $b \in \{1, \dots, 10\}$, let $\boldsymbol{\kappa}_{s_i, b}$ denote the curvature vector of the corresponding spatial slice, computed from a thresholded correlation graph. Each slice $\boldsymbol{\kappa}_{s_i, b}$ is assigned to the stage $\hat{s}$ whose reference curvature distribution minimizes the chosen discrepancy:
\[\hat{s}(s_i, b) = \arg \min_{j \in \{1, \dots, 5\} } \delta(\boldsymbol{\kappa}_{s_i, b}, \boldsymbol{\kappa}_{G_j}).\]
This is a nearest-neighbor rule in the space of curvature distributions, with the five reference distributions acting as class prototypes. The standardization conventions follow those described in \Cref{sec:geodesicclassification} for the geodesic method, with the reference distribution $\boldsymbol{\kappa}_{G_j}$ in place of the geodesic distribution $\boldsymbol{\kappa}_{d_\tau}$.

Using the 50 labeled observations, we compute classification accuracy as our evaluation metric:
\[\mathrm{Acc} = \frac{1}{50} \sum_{i=1}^5 \sum_{b=1}^{10} \mathds{1}[\hat{s}(s_i, b) = s_i].\]
Since each developmental stage has 10 observations and there are five developmental stages, a random classifier assigning stages uniformly at random would achieve expected accuracy of 0.2. We summarize this classification procedure in \Cref{alg:classifier}.

\begin{algorithm}[H]
    \caption{Slice-to-Stage Curvature Classification}
    \label{alg:classifier}
    \begin{algorithmic}[1]
        \Require{
        Real reference graphs $G_1, \dots, G_5$ at developmental stages $s_1 < \dots < s_5$; 
        
        spatial slices $\{(s_i, b)\}$ for $i = 1, \dots, 5$, $b = 1, \dots, 10$;
        
        discrepancy measures $\delta \in \{\delta_\mathrm{scalar}, \delta_W, \delta_{KL}\}$;
        
        correlation threshold $\theta > 0$
        }
        \Ensure{Predicted stages $\hat{s}(s_i, b)$, accuracy $\mathrm{Acc}$}

        \medskip
        \State \textbf{// Build reference curvature distributions}
        \For{each reference graph $G_j$, $j = 1, \dots, 5$}
            \State Compute $\boldsymbol{\kappa}_{G_j} := (\kappa_{x,y})_{\{x,y\}\in E_j}$; remove non-finite entries
        \EndFor

        \medskip
        \State \textbf{// Build spatial slice curvature distributions}
        \For{each stage $s_i$ and replicate $b = 1, \dots, 10$}
            \State Load correlation matrix; drop genes with fully missing data
            \State Threshold: retain edge $\{x,y\}$ iff $\rho_{xy} > \theta$, with weight $\rho_{xy}$
            \State Compute $\boldsymbol{\kappa}_{s_i, b} := (\kappa_{x,y})_{\{x,y\}\in E}$; remove non-finite entries
            \EndFor

            \medskip
            \State \textbf{// Classify each slice by nearest-neighbor assignment}
            \For{each slice $(s_i, b)$}
                \For{each reference stage $j = 1, \dots 5$}
                    \State Compute $\delta_{(i,b), j} \leftarrow \delta(\boldsymbol{\kappa}_{s_i, b}, \boldsymbol{\kappa}_{G_j})$
                \EndFor
                \State $\hat{s}(s_i, b) \leftarrow s_{j^*}$, where $j^* = \arg \min_j \delta_{(i,b), j}$
            \EndFor

            \medskip
            \State \textbf{// Evaluate classification performance}
            \State Acc $\leftarrow \frac{1}{50} \sum_{i=1}^5 \sum_{b=1}^{10} \mathds{1}[\hat{s}(s_i, b) = s_i]$
    \end{algorithmic}
\end{algorithm}

\subsection{Trend analysis with OR curvature}\label{subsection: trend analysis with OR curvature}

To assess how well the GW trajectory assesses developmental changes in average curvature, the following comparisons are considered:

\begin{itemize}
    \item[i)] \textbf{Stagewise predictive accuracy: observed stages versus endpoint geodesic.} Let $G_s$, $s=1,\ldots,5$, denote the observed weighted graphs at the five developmental stages, where $G_1$ corresponds to the initial embryonic stage and $G_5$ to the terminal larval stage. Let $\{t_s
    \}_{s=1}^{5} \in [0,1]$ be the corresponding developmental pseudo-times, derived from the reported developmental times \cite{tyler2000developmental}. We compare
\begin{equation*}
\left(\widetilde{\kappa}_{G_s}\right)_{s=1}^{5}
\qquad\text{and}\qquad
\left(\widetilde{\kappa}_{G_B^{w,t_s}}\right)_{s=1}^{5},
\end{equation*}
where $G_B^{w,t_s}$ is defined by the GW geodesic connecting $G_1$ and $G_5$ as described in \Cref{sec: curvature signatures},  and $\widetilde{\kappa}_G$ denotes the average edgewise OR curvature of a graph $G$.  We compute the relative $l_1$ error between these two curvature distributions, as well as the $R^2$ score of the GW-geodesic-based predictions $\left(\widetilde{\kappa}_{G_B^{w,t_s}}\right)_{s=1}^{5}$. 

\item[ii)] \textbf{Trajectory-level trend consistency: endpoint geodesic versus piecewise-stitched trajectory.}  The average curvatures of the graphs constructed from the GW geodesic from $G_1$ to $G_5$ are compared with those constructed from a piecewise GW path through all five stages. In particular, using the reported developmental times \cite{tyler2000developmental}, the interstage geodesics $G_1\to G_2$, $G_2\to G_3$, $G_3\to G_4$, and $G_4\to G_5$
are rescaled and concatenated to form a piecewise path $G_{\mathrm{pw}}^{w,t}$ on a common pseudo-time grid $\{t_j\}_{j=1}^{n}$ for some integer $n \geq 5$. The comparison is between
\begin{equation*}
\left(\widetilde{\kappa}_{G_B^{w,t_j}}\right)_{j=1}^{n}
\qquad\text{and}\qquad
\left(\widetilde{\kappa}_{G_{\mathrm{pw}}^{w,t_j}}\right)_{j=1}^{n}.
\end{equation*}
 We use the following sign alignment to quantify the directional consistency of average-curvature changes between the endpoint geodesic and the piecewise GW path.
    
\end{itemize}

\subsubsection{Sign alignment for scalar trends}
Following the directional symmetry criterion \cite{tay2001application}, we quantify agreement between the curvature trends in (ii) above using sign alignment. More generally, let $f,g:\mathbb{R}\to\mathbb{R}$ denote two scalar
trends sampled at $t_1<\cdots<t_n$.  Their sign-alignment score is defined by
\begin{equation}\label{eq: sign assignment}
S(f,g)
=
\frac{1}{n-1}
\sum_{i=1}^{n-1}
\mathbb{I}\bigl(
\operatorname{sign}(\Delta f(t_i))
=
\operatorname{sign}(\Delta g(t_i))
\bigr),
\end{equation}
where $\mathbb{I}$ denotes the indicator function and  $\Delta f(t_i) := f(t_{i+1}) - f(t_i)$.  To account for small oscillations, for $\epsilon>0$ we define the relaxed sign-alignment score by

\begin{equation}\label{eq: relaxed sign assignment - previous}
S_\epsilon(f,g)
=
\frac{1}{n-1}
\sum_{i=1}^{n-1}
\mathbb{I}\Bigl(
\operatorname{sign}(\Delta f(t_i))
=
\operatorname{sign}(\Delta g(t_i))
\ \text{or}\
\bigl(
|\Delta f(t_i)|<\epsilon
\ \text{or}\
|\Delta g(t_i)|<\epsilon
\bigr)
\Bigr).
\end{equation}
 We will choose $\epsilon$ as a percentage of maximum variation  $  \max \{ | \Delta f_i|,|\Delta g_i|: i =1,..., n \}.$ Note that $0\leq S, S_\epsilon\leq 1$.

\subsection{Trend analysis with dynamic OR curvature}\label{sec:dynamic-curvature-distributions}
 
Section 3.3 compares average OR curvature along the GW interpolations. Here we compare the edge-wise dynamic OR curvature distributions of the five observed stage-specific gene co-expression networks across diffusion parameters $\tau$. For each developmental stage, we compute the distance-scaled \footnote{To retain sensitivity to network changes reflected in the scale of graph weights, we use the distance-scaled version. For example, uniform rescaling leaves the original curvature unchanged, whereas the distance-scaled curvature changes with the scale.} dynamic OR curvature
\begin{equation}\label{eq:distance-scale curvature}
\kappa^{\mathrm{ds}}_{x,y}(\tau)
=d_G(x,y)\kappa_{x,y}(\tau)
=d_G(x,y)-W_1\!\left(\theta_x(\tau),\theta_y(\tau)\right).
\end{equation} 
The superscript $\mathrm{ds}$ distinguishes these values from the original dynamic OR curvature in \Cref{def: dynamic ORcurv}. We use 19 values of $\tau$ equally spaced on the $\log_{10}(\tau)$ scale between $0.01$ and $10$. Since the stage-specific networks have different retained edge sets, their curvature values are compared as empirical distributions rather than matched edges.

For each stage and value of $\tau$, we compute the median, interquartile range, and empirical cumulative distribution function (ECDF). Pairwise stage differences are measured using the 1-Wasserstein discrepancy $\delta_W$ from
\Cref{sec:geodesicclassification} between the distance-scaled curvature distributions. We compute this discrepancy for the raw curvature distributions and again after centering each distribution by its median and scaling by its interquartile range {\color{red} \footnote{The median and interquartile range provide robust measures of location and scale and are less sensitive to extreme observations than the mean and standard deviation.}}. The latter measures differences remaining after these location and scale adjustments.

For each pair of stages, we define the Wasserstein discrepancies averaged over $\log_{10}(\tau)$:  
\begin{equation}\label{eq:log avg W1}
    \delta_W^{\textrm{avg}}:= \frac{1}{\log_{10}\tau_{\max} - \log_{10}\tau_{\min}}
\int_{\tau_{\min}}^{\tau_{\max}}
\delta_W(\tau)\,{d\log_{10}(\tau)},
\end{equation}
where $\tau_{\min}=0.01$ and {$\tau_{\max}=10$}, and approximate this quantity
using the trapezoidal rule.
For consecutive developmental stages, we also record which transition has the largest discrepancy at each sampled value of $\tau$. These comparisons are descriptive and do not treat edge-level curvature values as independent biological replicates. The procedure is summarized in Algorithm \ref{alg:dynamiccurvature} and applied to the \textit{Drosophila} data in \Cref{subsec:dynamic-curvature-results}.

\begin{algorithm}[H]
    \caption{Comparison of Dynamic OR Curvature Distributions}
    \label{alg:dynamiccurvature}
    \begin{algorithmic}[1]
        \Require{
        Observed graphs $G_1, \dots, G_5$; diffusion parameters
        $T=\{\tau_1,\dots,\tau_M\}$; dynamic OR curvature distributions
        $\boldsymbol{\kappa}_{G_i}(\tau_j)$}
        \Ensure{
        Stage-level summaries, pairwise distributional comparisons, and
        summaries across diffusion scales}
        \medskip
        \State \textbf{// Summarize observed-stage curvature distributions}
        \For{each graph $G_i$ and $\tau_j \in T$}
            \State Verify that the curvature values are finite
            \State Compute the median, interquartile range, and ECDF of
            $\boldsymbol{\kappa}_{G_i}(\tau_j)$
        \EndFor 
        \medskip
        \State \textbf{// Compare stage-specific curvature distributions}
        \For{each pair of graphs $(G_i,G_k)$ and $\tau_j \in T$}
            \State Compute the raw Wasserstein discrepancy
            \State Median-center each distribution, divide by its
            interquartile range, and recompute the Wasserstein discrepancy
        \EndFor 
        \medskip
        \State \textbf{// Summarize differences across diffusion scales}
        \For{each pair of graphs $(G_i,G_k)$}
            \State Compute scale-averaged Wasserstein discrepancies over
            $\log_{10}(\tau)$
        \EndFor
        \For{each $\tau_j \in T$}
            \State Record which consecutive developmental transition has the
            largest raw and centered/scaled discrepancy
        \EndFor
    \end{algorithmic}
\end{algorithm}

\subsection{Quantifying developmental changes via COOT distance }\label{sec:coot_development_trajectory}
Each developmental stage $s \in \{1,\dots,S\}$ is represented by a rectangular gene--cell interaction matrix $\mathcal{H}_s \in \mathbb{R}^{n_c \times n_g}$, where $n_c$ denotes the number of spatial locations (cells) and $n_g$ denotes the number of genes. We interpret $\mathcal{H}_s$ as the discrete representation of a measure hypernetwork $(X_s,\mu_s, X_s',\mu_s',\omega_s)$, where $X_s$ is the set of cells and $X_s'$ is the set of genes. The corresponding empirical probability measures are defined as
$\mu_s = \frac{1}{n_c}\sum_{i=1}^{n_c}\delta_{x_i}$ and $\mu_s' = \frac{1}{n_g}\sum_{j=1}^{n_g}\delta_{x'_j}$. The hypernetwork function is given by
$\omega_s : X_s \times X_s' \rightarrow \mathbb{R}$,
and is discretely represented by the matrix $\mathcal{H}_s$, where each entry
$(\mathcal{H}_s)_{ij} = \omega_s(x_i, x'_j)$ encodes the expression level of gene $j$ at spatial location $i$.

Comparisons between developmental stages are performed using the Co-Optimal Transport (COOT) framework (Definition~\ref{Def:COOT}). The dissimilarity between any two stages $s$ and $s'$ is defined as $D_{ss'} = \mathrm{COOT}_p(\mathcal{H}_s, \mathcal{H}_{s'})$,
corresponding to the discrete hypernetwork $p$-distance. This distance is computed by jointly optimizing over couplings
$\pi \in \mathrm{Cpl}(\mu_s,\mu_{s'})$ on the cell spaces and
$\pi' \in \mathrm{Cpl}(\mu'_s,{\mu'}_{s'})$ on the gene spaces, using block coordinate descent (see \cite[Algorithm~1]{Redko_2020}). The optimization simultaneously aligns both the spatial (cell) and molecular (gene) structures, yielding a transport-based discrepancy that captures differences in their joint relational organization rather than entrywise comparisons. The resulting pairwise dissimilarities form a symmetric matrix $D \in \mathbb{R}^{S \times S}$. Classical multidimensional scaling (MDS) is then applied to obtain one-dimensional embeddings $y_1,\dots,y_S \in \mathbb{R}$ that preserve the COOT-induced geometry of developmental stages.

\section{Application to \textit{Drosophila} data}\label{sec:application}
\begin{figure}
    \centering
    \includegraphics[width=1.0\linewidth]{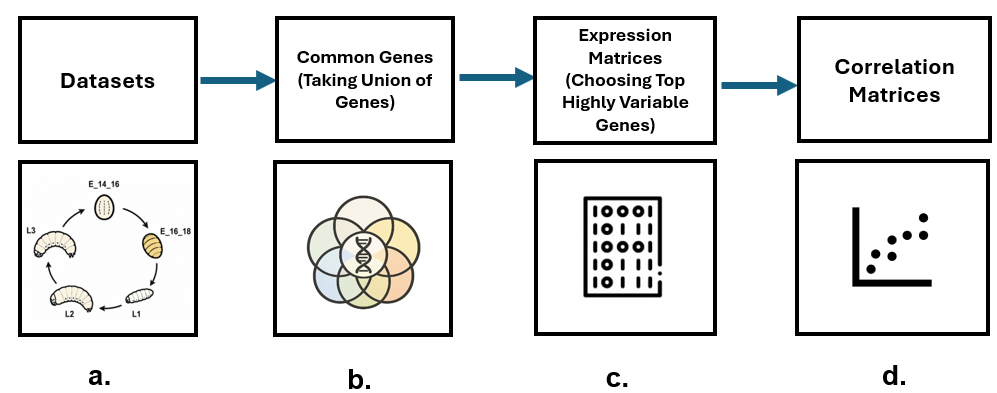}
    \caption{Flow chart on the steps involved for pre-processing the \textit{Drosophila} dataset. }
    \label{fig:flow_chart}
\end{figure}

 \subsection{Datasets and preprocessing for gene co-expression} \label{sec:data preprocessing}
\begin{figure}[!htpb]
    \centering
    
    \begin{minipage}{0.45\linewidth}
        \centering
        \includegraphics[width=\linewidth]{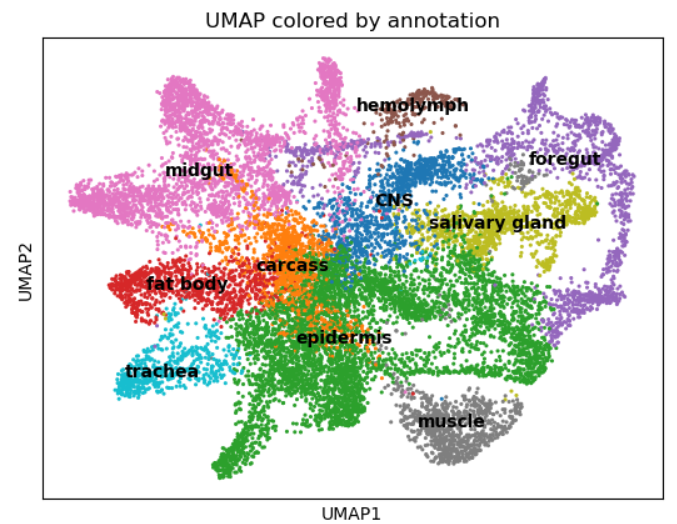}
        \caption{Cluster plot obtained via Uniform Manifold Approximation and Projection (UMAP) on \textit{Drosophila} gene data. Clusters are associated with tissue type.}
        \label{fig:ffUMAP}
    \end{minipage}
    \hfill
    \begin{minipage}{0.45\linewidth}
        \centering
        \includegraphics[width=\linewidth]{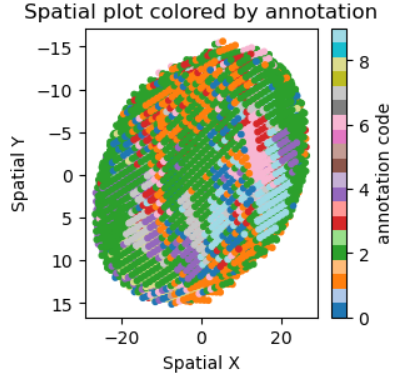}
        \caption{Plot of spatial locations of tissue samples from \textit{Drosophila} data, color coded by tissue type.}
        \label{fig:ffSpatial}
    \end{minipage}
\end{figure}
This study uses spatially resolved single-cell transcriptomic data from \textit{Drosophila} covering five developmental stages (E14--16h, E16--18h, L1, L2, and L3), generated using Stereo-seq \cite{wang2025drosophila}. The data provide gene expression measurements at single-cell or spot resolution together with spatial coordinates, enabling joint analysis of transcriptional variation and spatial organization across development (see Figures \ref{fig:ffUMAP} and \ref{fig:ffSpatial} for a low-dimensional view of the dataset). Each tissue section is associated with a unique slice ID, ensuring that all observations within a section share a consistent spatial reference frame. We now describe the pre-processing steps needed to construct the graphs which serve as inputs for the GW geodesics and measure hypernetworks which serve as inputs for the COOT distance. A summary figure on the intermediate steps is shown in Figure \ref{fig:flow_chart}. We refer to Section \ref{sec:add_data_prepos} for a detailed information on the additional pre-processing in the generation of the gene-gene correlation matrices for subsequent learning tasks such as geodesics inference and classification. 
Throughout, graphs are handled numerically using the networkx package in Python (see \cite{hagberg2008exploring}), and Gromov-Wasserstein-related computations are handled using functions in the Python Optimal Transport library (see \cite{flamary2021pot}). To handle some Ollivier-Ricci curvature computations, we additionally use code from \cite{Unfolding_the_m_Gosztolai2021}.

\paragraph{Gene selection.}
To construct a consistent feature space across stages, we first identify the top 50 highly variable genes to reduce the dimensionality of the dataset within each developmental stage based on expression variance across cells or spots \cite{dasgupta2025scanpy}. The final gene set is obtained by taking the union of these stage-specific lists. This ensures that genes showing strong variability in at least one developmental stage are retained, while enforcing a fixed set of genes across all stages for comparison.

\paragraph{Gene co-expression networks.}
For each developmental stage, we restrict the expression matrix to the unified gene set and compute pairwise Pearson correlation coefficients between genes. This produces a stage-specific correlation matrix encoding co-expression relationships.

To construct a sparse gene interaction network, we retain only the strongest correlations (top 5\% in absolute value) and set all remaining entries to zero. The resulting matrix is interpreted as a weighted adjacency matrix of a gene co-expression graph. Each correlation matrix was thresholded by retaining only entries whose magnitude exceeded a threshold, $\mu$; that is, entries satisfying $|a_{ij}| \geq \mu$ were preserved. We fix the threshold hyperparameter $\mu=0.2$ throughout the experiments. A systematic selection procedure based on spectral graph theory and maximizing ``nearly disconnected" components can be found in \cite{perkins2009threshold}, though adapting threshold selection to curvature-based analysis is left for future investigation. To avoid issues arising from negative weights in geometric constructions, the correlation values are linearly rescaled from $[-1,1]$ to $[0,1]$ following \cite{Graph_Curvature_Sandhu2015}. Since all stages use the same gene set and ordering, the resulting graphs are directly comparable across developmental time.

\paragraph{Hypernetwork representation and COOT inputs.}
We refer to the resulting stage-wise weighted gene graphs as hypernetwork representations, in the sense that they encode structured higher-order dependencies induced by gene co-expression rather than independent pairwise interactions alone.
For Co-Optimal Transport (COOT) (see, \cite[Algorithm 1]{Redko_2020} for computational details), we additionally incorporate the original gene expression matrices together with spatial coordinates. This yields, for each stage, a coupled representation consisting of gene expression profiles and spatial organization. COOT then operates on these coupled structures to simultaneously align gene-level organization and spatial transcriptomic structure across developmental stages.

\subsection{Temporal ordering and classification of gene expression networks}\label{subsec:classification}

Using the distributions of edge-wise Ollivier-Ricci curvature in both the GW geodesics and observed \textit{Drosophila} data, we want to use curvature to determine the relative ordering of the measured data using the methods described in \Cref{sec:geodesicclassification} and \Cref{subsubsec:realcurvaturecomp}. The \textit{Drosophila} data is processed using the methods described in \Cref{sec:data preprocessing} and \Cref{sec:add_data_prepos}. Specifically, the method described in \Cref{sec:data preprocessing} obtains five gene expression networks, one for each developmental stage, where the initial and terminal networks are used to construct 40 Gromov-Wasserstein geodesics that are used to perform rank ordering as described in \Cref{sec:geodesicclassification} as well as function as the class prototypes for the classification procedure described in \Cref{subsubsec:realcurvaturecomp}, while the data preprocessing described in \Cref{sec:add_data_prepos} yields 50 gene expression networks corresponding to different spatial slices that function as test data for both the rank ordering algorithm and the classification algorithm.

We first present the results of the rank ordering prediction algorithm from \Cref{sec:geodesicclassification}. The results for the two best-performing methods (Wasserstein and KL divergence) are shown in \Cref{fig:KLclassifiation} and \Cref{fig:Wclassification}. Overall, the the KL divergence and Wasserstein methods recover the right qualitative trend in the data, as indicated by the high Kendall and Spearman rank correlations of $\tau = 0.9016$ and $\rho = 0.9528$ for the KL divergence method and $\tau = 0.8889$ and $\rho = 0.9398$ for the Wasserstein method. A visualization for the method using mean curvature as the discrepancy measure, whose performance was subpar compared to the Wasserstein and KL divergence methods, is shown in \Cref{fig:meangeoclass} in \Cref{appendix:suppfigs}. These results suggest that the edge-wise Ollivier-Ricci curvature distributions of Gromov-Wasserstein geodesics provide descriptive information about the developmental stages and reveal the overall developmental trajectory, illustrating that the curvature distributions of interpolated metric measure spaces obtained via Gromov-Wasserstein geodesics can predict the relative ordering of \textit{Drosophila} developmental stages, including intermediate stages between the initial and terminal points.

These results reflect similar findings in \cite{baptista2024}, which also successfully recovered the temporal ranking of time series data using Ricci flow. In a similar fashion to our geodesic construction, Ricci flow trajectories were constructed between an initial and a terminal phase of gene expression networks. This Ricci flow trajectory model successfully recovered the temporal ordering of embryonic stem cell differentiation at six time points and human myoblast differentiation at eight time points. Likewise, the experimental results in \cite{Graph_Curvature_Sandhu2015} established that cancerous networks can be identified based on their higher average curvature compared to normal cellular networks. Our rank ordering predictions using the Ollivier-Ricci curvature distributions of Gromov-Wasserstein geodesics further support the use of OR curvature as a discriminative biological marker. Furthermore, our ability to recover these markers using only interpolation via Gromov-Wasserstein geodesics between initial and terminal gene expression networks suggests that Gromov-Wasserstein geodesics can function as effective models of real biological data if intermediate temporal measurements are unavailable.
\begin{figure}[!htpb]
    \centering
    \begin{subfigure}[b]{\textwidth}
        \centering
        \includegraphics[scale=0.4]{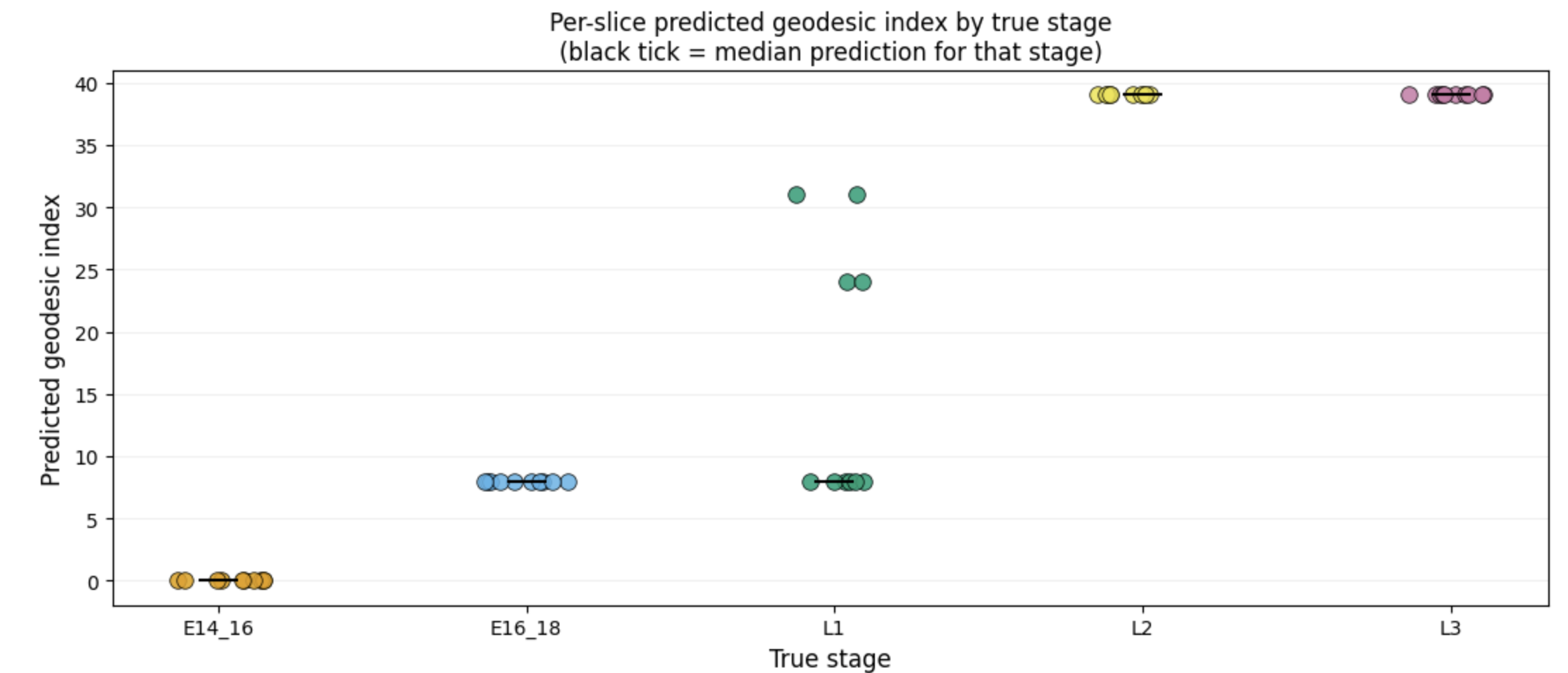}
        \caption{Rank ordering results for KL divergence method.}
        \label{fig:KLclassifiation}
    \end{subfigure}

    \begin{subfigure}[b]{\textwidth}
        \centering
        \includegraphics[scale=0.4]{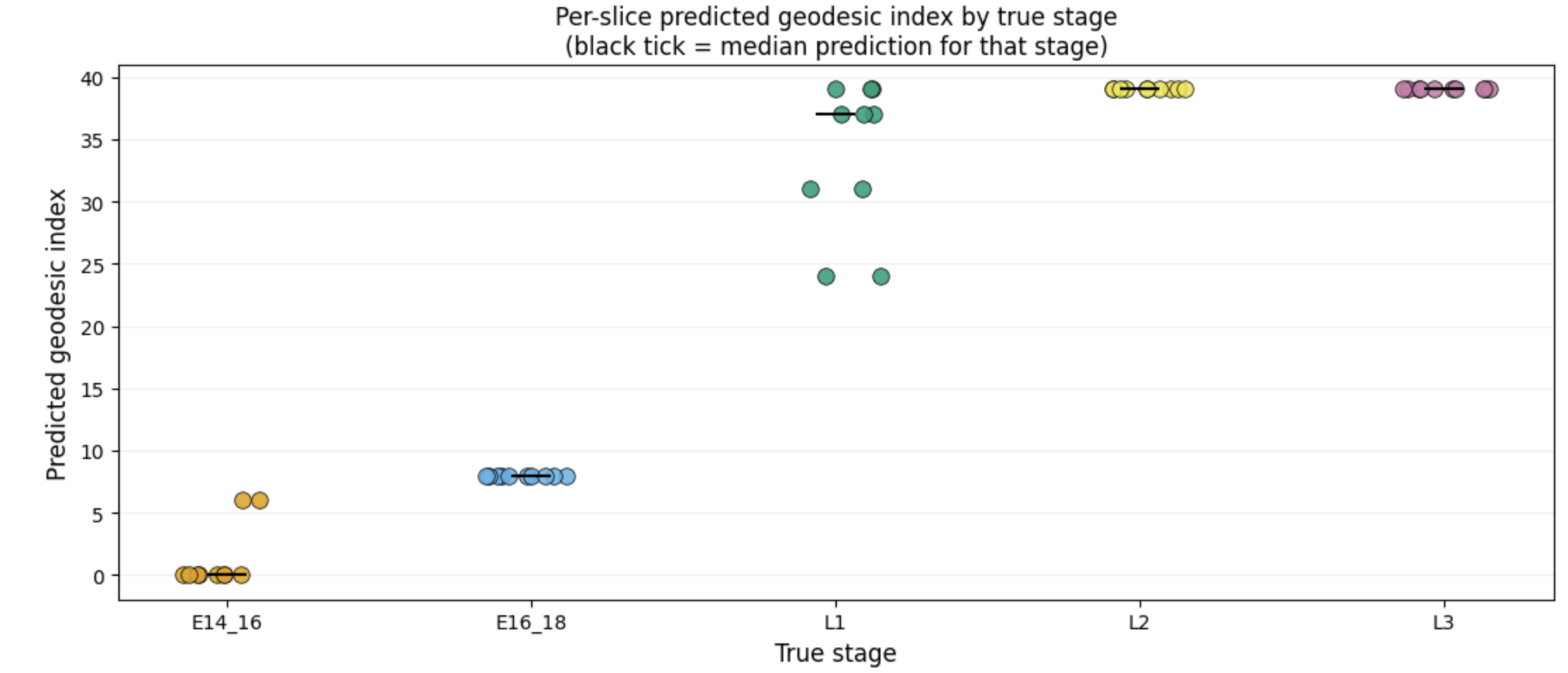}
        \caption{Rank ordering results for Wasserstein method.}
        \label{fig:Wclassification}
    \end{subfigure}
    \caption{Rank ordering classification results for the KL divergence and Wasserstein methods.}
    \label{fig:classification}
\end{figure}

We next present the classification results from the methods described in \Cref{subsubsec:realcurvaturecomp}, shown in \Cref{tab:classificationresults}. The Wasserstein method obtained perfect classification accuracy, correctly classifying all spatial slices into the correct developmental stage. In \Cref{fig:discrepancy}, the per-slice, per-reference-stage discrepancy values for the Wasserstein and KL divergence methods are visualized as a heatmap, showing the full decision surface to indicate not only how each slice was classified but also how close other discrepancy measures were. We show the plots for the mean and median methods in \Cref{fig:meanmeddiscrep} in \Cref{appendix:suppfigs} (with the mean method shown in \Cref{subfig:meandiscrepplot} and the median method shown in \Cref{subfig:meddiscrepancy}). The discrepancy values for the Wasserstein method shown in \Cref{fig:Wassdiscrepancy} indicate that all slices were safely classified in their correct developmental stage. For KL divergence, shown in \Cref{fig:KLdiscrepancy}, we can see that the three incorrectly-classified slices were still close in discrepancy measure to the correct classification. Both plots also indicate that the class with the next smallest discrepancy measure after the closest one is often temporally adjacent to the correct class. For example, in \Cref{fig:Wassdiscrepancy}, we can see that for the E14\_16 spatial slices, which were all correctly classified as class E14\_16, class E16\_18 is the next smallest in Wasserstein discrepancy, while discrepancy measurements in the larval stages L1, L2, and L3 are much larger, suggesting that temporally adjacent stages may have more similar curvature distributions compared to earlier or later stages. Overall, these results demonstrate that using the curvature distributions of real prototypes for each developmental stage has strong predictive power to classify gene expression networks into their correct temporal stage.
\begin{table}[!htpb]
    \centering
    \caption{Classification results by comparing the curvature distributions with the common gene space gene expression networks for each developmental stage to the spatial slices at each developmental stage.}
    \label{tab:classificationresults}
    \begin{tabular}{l|l|l|l|l}
    \toprule
    Method        & Accuracy & Number Correct & Kendall $\tau$ & Spearman $\rho$ \\
    \midrule
    Wasserstein   & 100.00\% & 50             & 1.0000         & 1.0000          \\
    KL Divergence & 94.00\%  & 47             & 0.9735         & 0.9867          \\
    Scalar Mean   & 75.00\%  & 37             & 0.7098         & 0.8073          \\
    Scalar Median & 60.00\%  & 30             & 0.3162         & 0.4617   \\      
    \bottomrule
    \end{tabular}
\end{table}
\begin{figure}[!htpb]
    \centering
    \begin{subfigure}[b]{0.48\textwidth}
        \centering
        \includegraphics[scale=0.45]{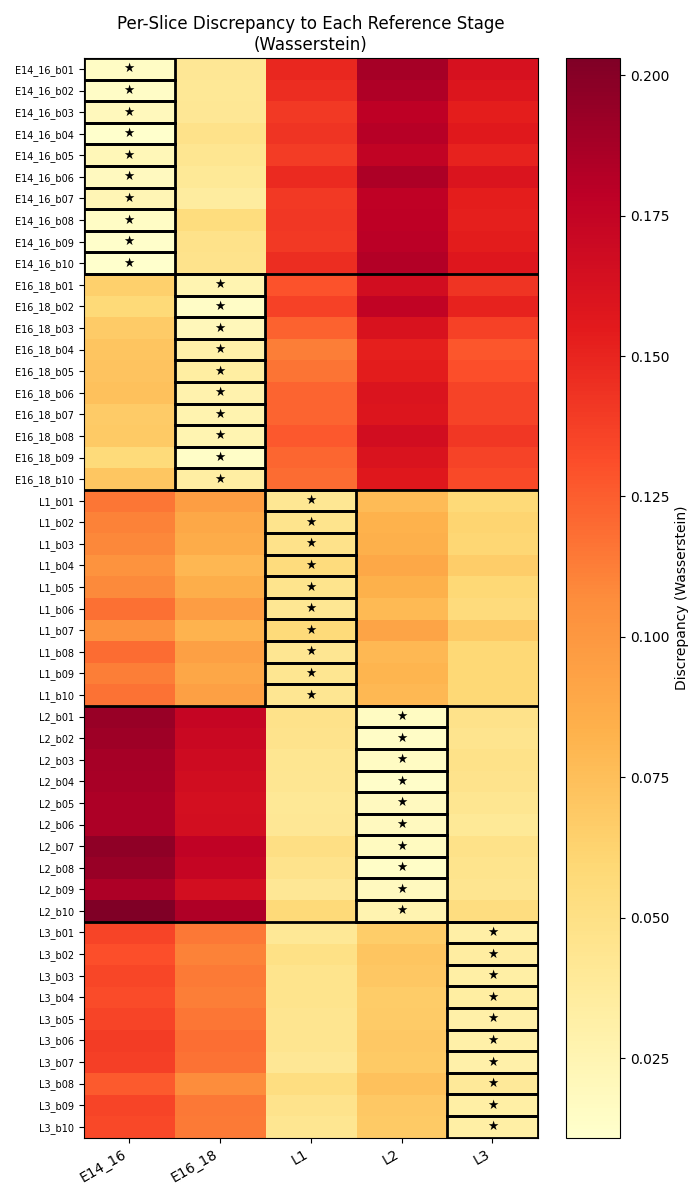}
        \caption{Wasserstein.}
        \label{fig:Wassdiscrepancy}
    \end{subfigure}
    \begin{subfigure}[b]{0.48\textwidth}
        \centering
        \includegraphics[scale=0.45]{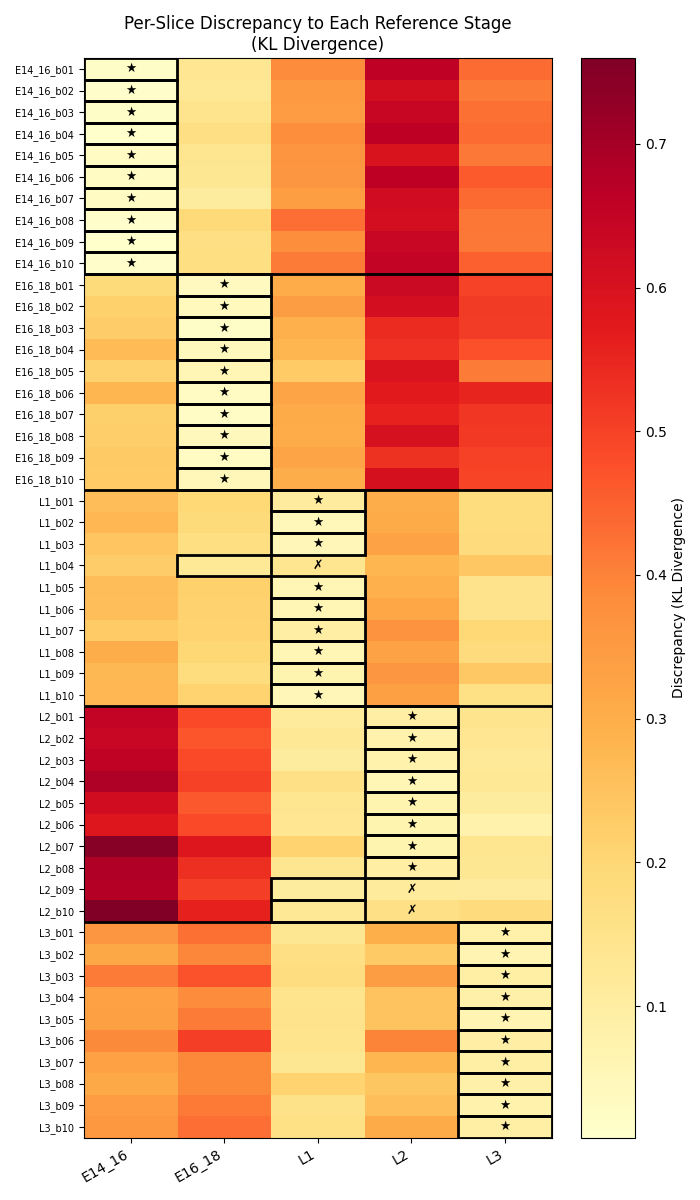}
        \caption{KL divergence.}
        \label{fig:KLdiscrepancy}
    \end{subfigure}
    \caption{Plot of per-slice, per-reference-stage discrepancy values for the Wasserstein and KL divergence classification methods. Correctly-classified slices are marked with a star symbol and a black box, while incorrectly-classified slices are denoted with an x symbol, with a black box around the class into which it was incorrectly classified.}
    \label{fig:discrepancy}
\end{figure}

\subsection{Developmental trends using GW Geodesics and Ollivier--Ricci curvature}\label{subsec: geodesics and curvature}

Using the data pre--processing detailed in \cref{sec:data preprocessing}, we apply the method detailed in \cref{subsection: graph approximations} using parameters $\gamma = 10^4$ and $\delta  = \frac{1}{4}$. To justify the validity of the method, we first compute the geodesic curvatures between each developmental stage of the fruit fly individually (e.g. E\_14\_16 to E\_16\_18, then E\_16\_18 to L1, and so forth). Then, we compare these individual results to computing the geodesic curvature between the first embryo stage and the final larval stage. In effect, this procedure compares the real temporal trajectory of the fruit fly's development with the interpolated path between the initial and final temporal stages.

\Cref{fig:gene networks 0.2 threshold} visualizes these gene correlation networks after applying this threshold; for a visualization of the heatmaps of these networks, see \Cref{fig:correlations} in \Cref{appendix:suppfigs}. We consider barycenters corresponding to $t\in \{0, \frac{1}{10}, \frac{2}{10}, \dots , 1\}$ along each interstage geodesic and computer their curvatures, with results shown in \cref{fig:fruit fly 0.2 graph threshold}.  We then compute a stitched trajectory by joining the interstage geodesics and rescaling them according to the reported development times (15 hours, 17 hours, 24 hours, 49 hours and 73 hours) \cite{tyler2000developmental} for stages E14-16, E16-18, L1, L2, and L3, respectively. This relative scaling yields a total of 41 barycenters, from which we compute the average curvature along the stitched trajectory from E14-16 and Larvae 3, shown in \cref{fig: stiched geodesic overlayed}. 

\begin{figure}[!htpb]
    \centering
    \begin{subfigure}{0.3 \linewidth}
        \includegraphics[width=\linewidth]{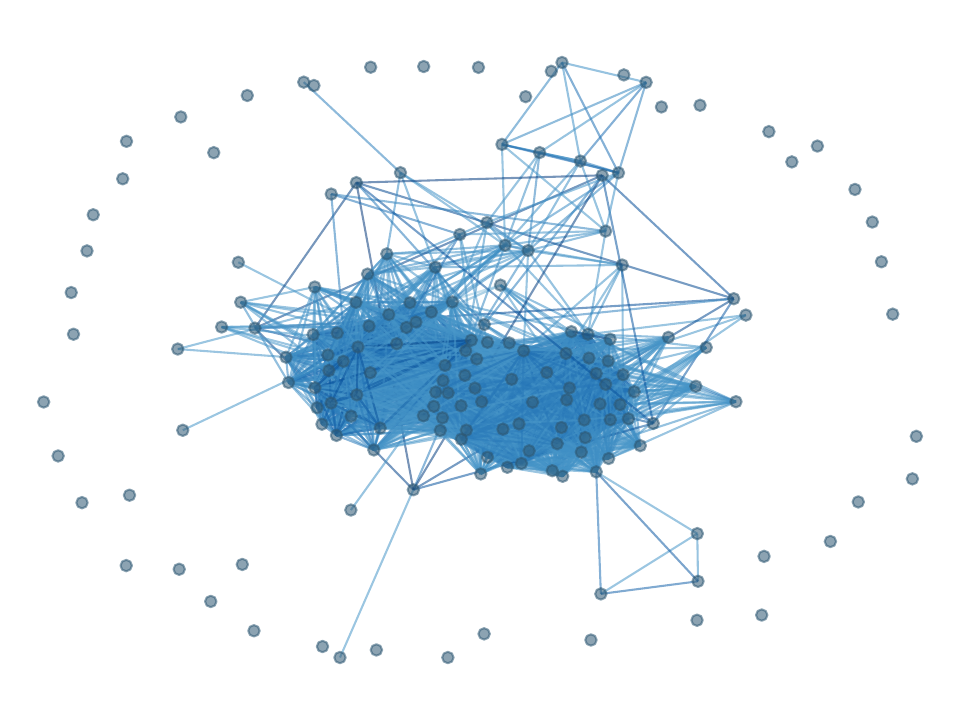}
        \caption*{E14-16}
    \end{subfigure}
    \begin{subfigure}{0.3 \linewidth}
        \includegraphics[width=\linewidth]{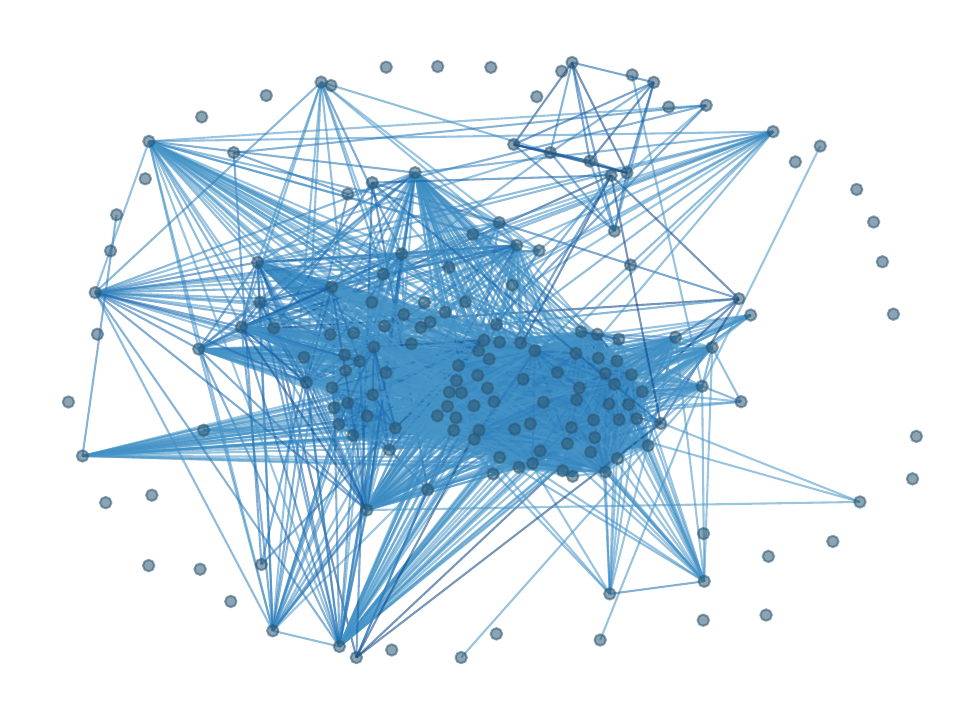}
        \caption*{E16-18}
    \end{subfigure}
    \begin{subfigure}{0.3 \linewidth}
        \includegraphics[width=\linewidth]{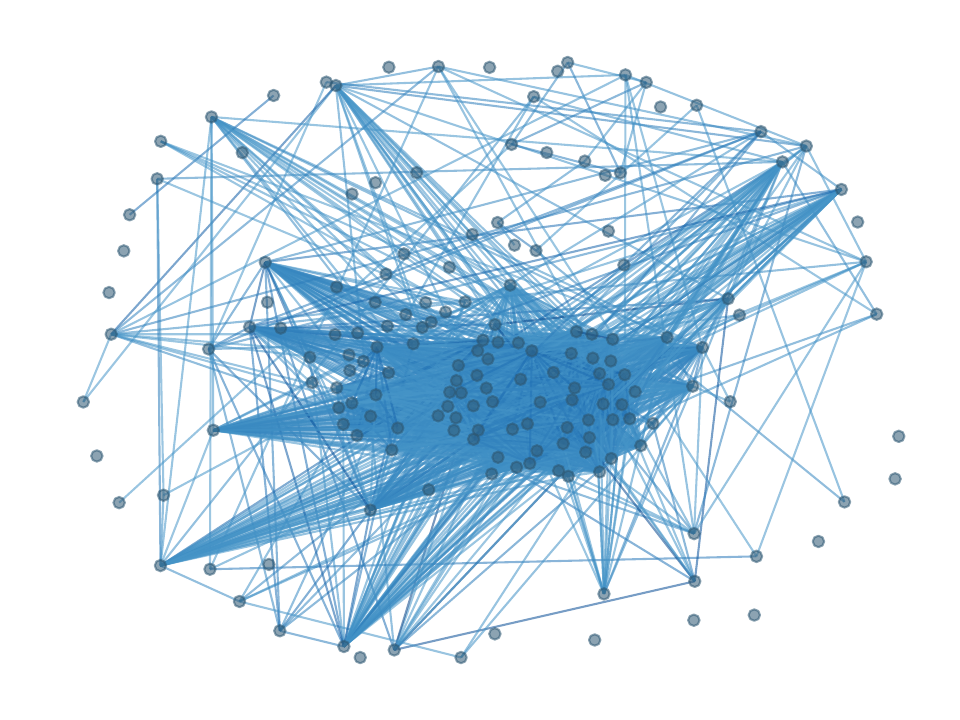}
        \caption*{L1}
    \end{subfigure}
    \\
    \begin{subfigure}{0.3 \linewidth}
        \includegraphics[width=\linewidth]{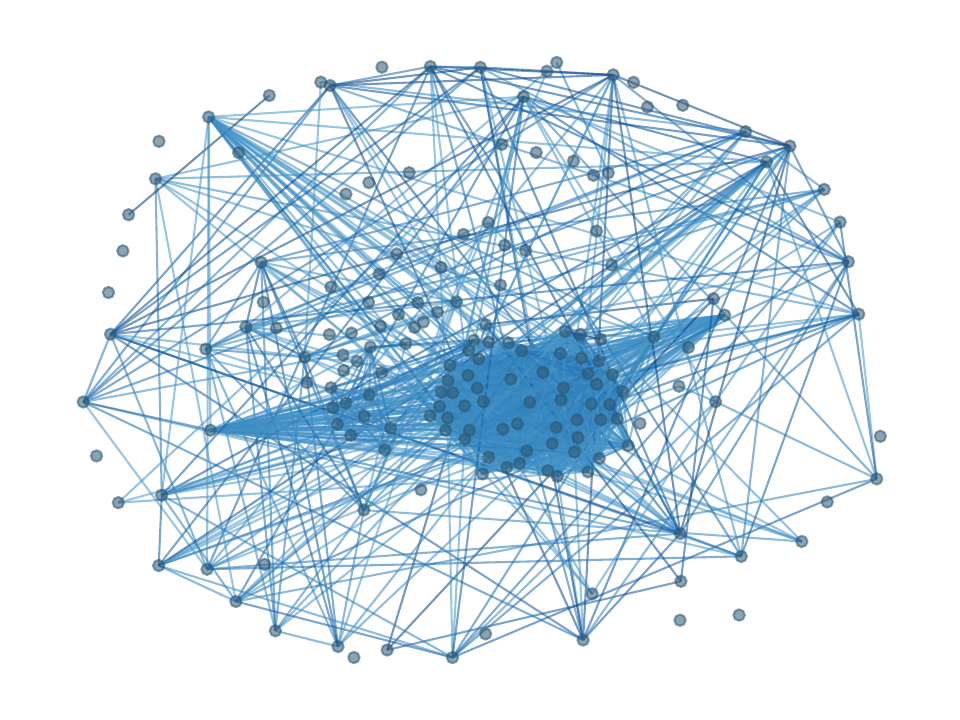}
        \caption*{L2}
    \end{subfigure}
    \begin{subfigure}{0.3 \linewidth}
        \includegraphics[width=\linewidth]{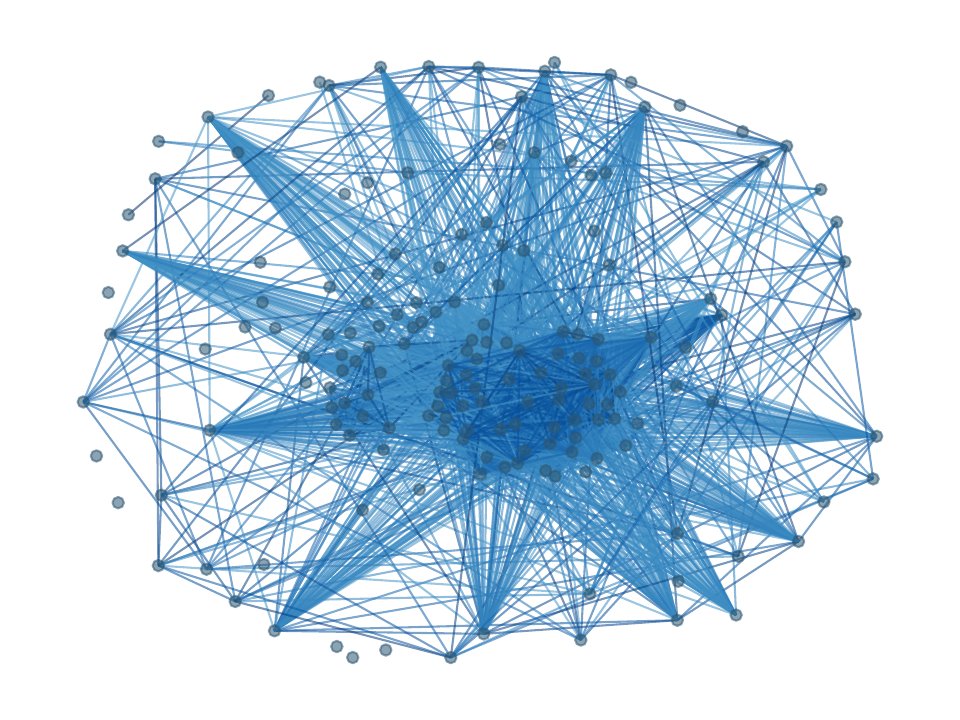}
        \caption*{L3}
    \end{subfigure}
    \caption{Comparison of gene expression networks with a thresholding of $\mu = 0.2$.}
    \label{fig:gene networks 0.2 threshold}
\end{figure}

{ To evaluate the predictive performance of the endpoint GW geodesic in comparison (i) outlines in \Cref{subsection: trend analysis with OR curvature}, the relative $l_1$ error is computed to measure the overall discrepancy between the observed and geodesic-based curvatures:
\begin{equation*}
\frac{
\sum_{s=1}^{5}
\left|
\widetilde{\kappa}_{G_s}
-
\widetilde{\kappa}_{G_B^{w,t_s}}
\right|
}{
\sum_{s=1}^{5}
\left|
\widetilde{\kappa}_{G_s}
\right|
} \times 100\% = 6.24\%.
\end{equation*}
The values for these average edgewise OR curvatures are listed in \cref{tab: average OR curvature values} in \cref{appendix:suppfigs}. Additionally, to assess the extent to which the GW-geodesic-based curvatures reflect the observed variation across the five development stages, we also compute the $R^2$ score

\begin{equation*}
R^2
=
1-
\frac{
\sum_{s=1}^{5}
\left(
\widetilde{\kappa}_{G_s}
-
\widetilde{\kappa}_{G_B^{w,t_s}}
\right)^2
}{
\sum_{s=1}^{5}
\left(
\widetilde{\kappa}_{G_s}
-
\frac{1}{5}\sum_{r=1}^{5}\widetilde{\kappa}_{G_r}
\right)^2
} = \fpeval{round(0.375138393776545, 3)}.
\end{equation*}
Hence, across the five developmental stages, the GW-geodesic-based predictions reduce the square error by $37.5\%$  relative to the constant across-stage mean predictor $\frac{1}{5}\sum_{r=1}^{5}\widetilde{\kappa}_{G_r}$. 
}

 The average-curvature trajectory along the endpoint GW geodesic from E14--16 to Larvae 3 is then compared with that along the piecewise stitched GW path in \cref{fig: stiched geodesic overlayed}, addressing comparison (ii) from \Cref{subsection: trend analysis with OR curvature}. The average-curvature trajectory along the piecewise GW path appears visually similar to that along the endpoint GW geodesic, which does not use any intermediate stages. This suggests that endpoint GW geodesic can potentially serve as a useful approximation to the developmental average-curvature trajectory, and more broadly, may capture information about the developmental progression reflected in the curvature networks.  This visual observation is assessed numerically using the sign assignment procedure described in Equation \eqref{eq: sign assignment} and \cref{eq: relaxed sign assignment - previous}, with $n=41$ pseudo-time points. The scores take values in $[0,1]$, with values closer to $1$ indicating stronger sign alignment.  As reported in \Cref{tab: Sign Assigments}, the sign-alignment score is $0.72$, substantially above the chance baseline of $0.5$ under independent, equally likely positive and negative curvature changes. 
 To reduce the influence of the small oscillations potentially due to noise, the relaxed sign alignment scores are also reported, with $\epsilon$ chosen as a percentage of the maximum variation. For each $\epsilon$, the score is compared with the empirical relaxed chance baseline (cf. \Cref{lm: relaxed chance baseline} in \Cref{appendix:relaxed baseline}):
 \begin{equation*}
\widehat{C}_\epsilon
=
1-\frac{1}{2}
\bigl(1-\widehat{p}_{f,\epsilon}\bigr)
\bigl(1-\widehat{p}_{g,\epsilon}\bigr),
\end{equation*}
where
\begin{equation*}
\widehat{p}_{f,\epsilon}
=
\frac{1}{n-1}
\sum_{i=1}^{n-1}
\mathbb{I}\bigl(|\Delta f_i|<\epsilon\bigr),
\qquad
\widehat{p}_{g,\epsilon}
=
\frac{1}{n-1}
\sum_{i=1}^{n-1}
\mathbb{I}\bigl(|\Delta g_i|<\epsilon\bigr),
\end{equation*}
and $f,g:\mathbb{R}\to\mathbb{R}$ denote two scalar
trends sampled at $t_1<\cdots<t_n$.
For example, when $\epsilon=0.008$, the relaxed sign-alignment score is $S_\epsilon=0.82$, which is higher than the corresponding empirical chance baseline $\widehat{C}_\epsilon=0.73$ and indicates directional consistency between the endpoint geodesic and the more refined piecewise trajectory.

\begin{table}
\centering

\caption{Sign Assignment {($S$)} and Relaxed Sign Assignment { ($S_\epsilon$)} comparing the geodesic from E14-16 to L3  and the piecewise stitched geodesic for \cref{subsec: geodesics and curvature} with $\epsilon$ scaled to 5.0\% ($\epsilon = \fpeval{round(0.004096011001264638,4)} )$ and 10.0\% of maximum variation ($\epsilon = \fpeval{round(0.008192022002529276,4)}$).}
\label{tab: Sign Assigments}

\begin{tabular}{lrr}
\toprule
 & Score &  \quad \quad Chance baseline \\
\midrule
$S$ & 0.72 & 0.50 \\
$S_\epsilon \, (\epsilon = \fpeval{round(0.004096011001264638,4)})$ 
& 0.75 & \fpeval{round(0.6609375,2)}  \\
$S_\epsilon \, (\epsilon = \fpeval{round(0.008192022002529276,4)})$ 
& 0.82 &  0.73 \\
\bottomrule
\end{tabular} 
\end{table}

\begin{figure}[!htbp]
    \centering
\includegraphics[width=0.6\linewidth]{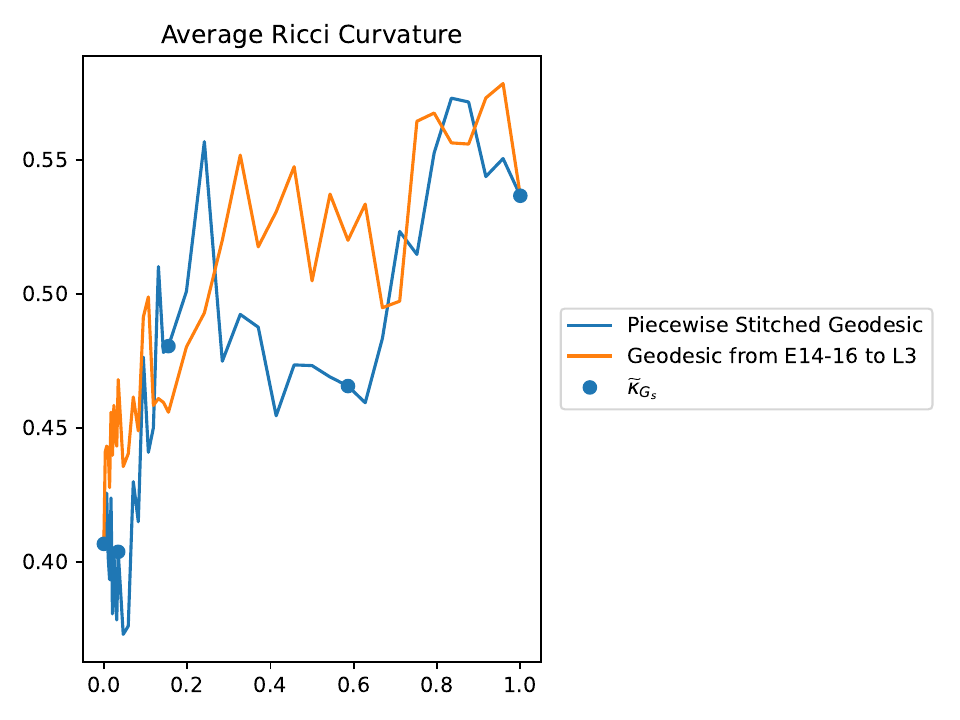}
    \caption{Average OR curvatures for the geodesics computed between E14-16 and Larvae 3 and the stitched geodesic as detailed in \cref{subsec: geodesics and curvature}. $\widetilde{\kappa}_{G_s}$, $s=1,\ldots,5$, denote the observed weighted graphs at the five developmental stage. }
    \label{fig: stiched geodesic overlayed}
\end{figure}

\subsection{Developmental trends using dynamic Ollivier--Ricci curvature}\label{subsec:dynamic-curvature-results}
 
In this section we  compare the distributions of edge-wise distance-scaled dynamic OR curvature \eqref{eq:distance-scale curvature} associated with the five observed stage-specific networks as described in \Cref{sec:dynamic-curvature-distributions}. Because the retained edge sets differ across stages, we compare empirical curvature distributions rather than matched edges. The largest raw 1-Wasserstein discrepancy, { denoted by $\delta_W$}, among consecutive stages depends on diffusion scale $\tau$. L1 to L2 has the largest $\delta_W$ at 12 of the 19 sampled values of $\tau$ ($0.01$--$0.681$), while L2 to L3 is largest at the remaining seven ($1$--$10$). After averaging over $\log_{10}(\tau)$ (cf. \eqref{eq:log avg W1} ), L1 to L2 also has the largest raw discrepancy, $\delta_W^{\textrm{avg}}$. In contrast, after median centering and interquartile-range scaling, the resulting discrepancy, denoted by $\delta_{W, cs}$, is largest for L2 to L3  at all 19 values of $\tau$. Its averaged centered/scaled discrepancy, $\delta_{W,\mathrm{cs}}^{\mathrm{avg}}$, is $0.353$,  compared with at most $0.244$ for the other consecutive transitions. Table \ref{tab:dev-curvature} summarizes these results.

\begin{table}[!htbp]
    \centering
    \small
    \caption{Dynamic OR curvature comparisons for consecutive developmental stages. The Wasserstein discrepancies are averaged over $\log_{10}(\tau)$. The third and fifth columns report the number of the 19
    sampled diffusion parameters at which each transition has the largest corresponding  discrepancies,  $\delta_{W}$ and $\delta_{W,\mathrm{cs}}$ respectively,  among the four consecutive transitions.}
    \label{tab:dev-curvature}
    \begin{tabular}{lcccc}
        \toprule
        Transition
        & {\shortstack{Raw\\$\delta_W^{\textrm{avg}}$}}
  
           & {\shortstack{$\tau$ counts with largest\\$\delta_{W}$}}
        & \shortstack{Centered/scaled\\$\delta_{W,\mathrm{cs}}^{\mathrm{avg}}$}
        & {\shortstack{$\tau$ counts with largest\\$\delta_{W,\mathrm{cs}}$}} \\
        \midrule
        E14--16 to E16--18 & 0.125 & 0/19  & 0.133 & 0/19 \\
        E16--18 to L1      & 0.158 & 0/19  & 0.193 & 0/19 \\
        L1 to L2           & 0.195 & 12/19 & 0.244 & 0/19 \\
        L2 to L3           & 0.092 & 7/19  & 0.353 & 19/19 \\
        \bottomrule
    \end{tabular}
\end{table}

\begin{figure}[!htbp]
    \centering
    \includegraphics[width=\textwidth]{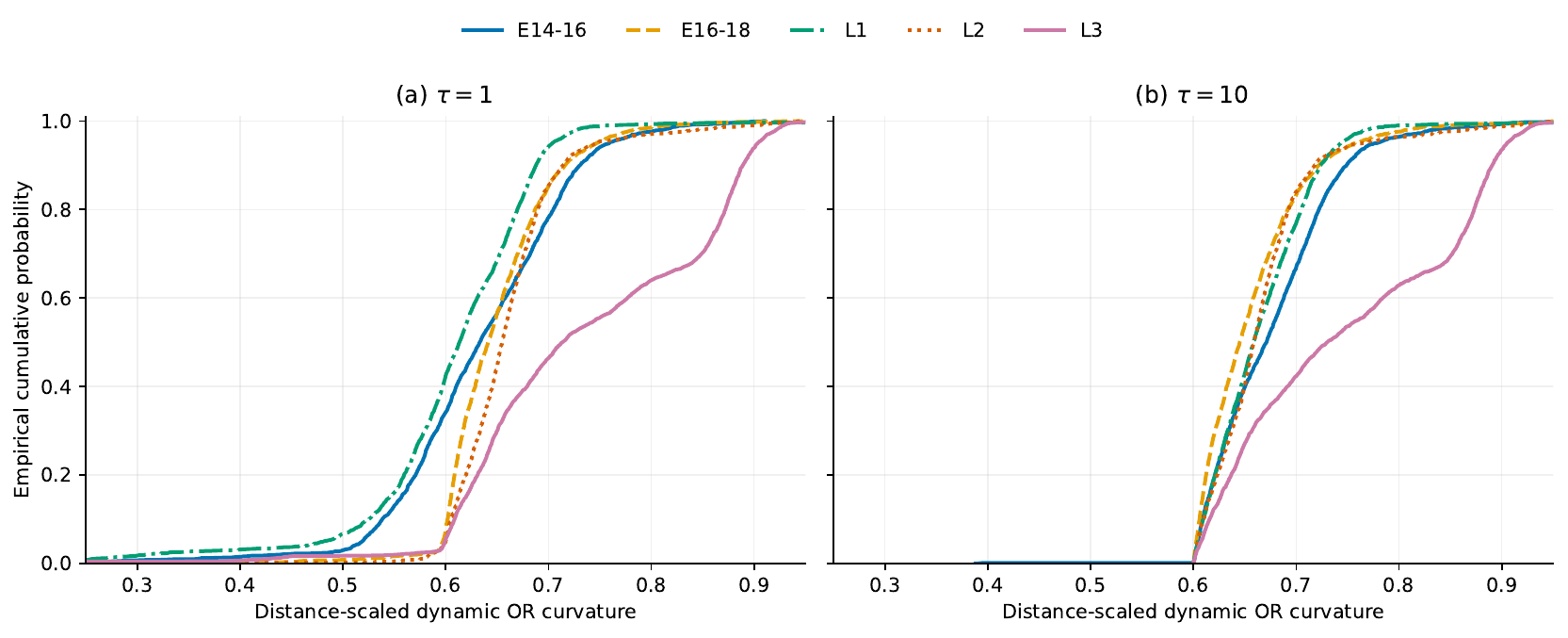}
    \caption{Empirical cumulative distribution functions of the edge-wise distance-scaled dynamic OR curvature values for the five observed developmental-stage networks at $\tau=1$ and $\tau=10$. The common x-axis range $[0.25,0.95]$ contains the central 98\% of the curvature values for each developmental stage at both $\tau=1$ and $\tau=10$.}
    \label{fig:dynamic-curvature-ecdf-main}
\end{figure}

At larger diffusion scales, L3 is consistently separated from the earlier stages. For every sampled diffusion scale from $\tau=1$ through $\tau=10$, each raw 1-Wasserstein ($\delta_W$) comparison involving L3 is larger than every comparison among E14--16, E16--18, L1, and L2. The ECDFs in Figure \ref{fig:dynamic-curvature-ecdf-main} show that L3 also has the largest interquartile range at $\tau=1$ and $\tau=10$. This pattern is compatible with the larger COOT distances involving L3 in \Cref{subsec:coopt_hypernetworks}. These results indicate that dynamic OR curvature distributions capture developmental-stage differences across diffusion scales, complementing the curvature-based classification results in \Cref{subsec:classification}.

{ A possible hypothesis for the larger-$\tau$ behavior (cf. column 3 of \Cref{tab:dev-curvature}) is that widespread transcriptional changes during L3 induce co-expression changes across broader portions of the gene network. This is consistent with the significant remodeling and widespread shift in gene-expression profiles reported by Wang et al.~\cite{wang2025drosophila} in their study of the \textit{Drosophila} midgut; see also \Cref{subsec:coopt_hypernetworks}. Since larger diffusion scales incorporate co-expression changes across broader neighborhoods \cite{Unfolding_the_m_Gosztolai2021}, they may better reflect the structural discrepancy between L2 to L3, in line with this transition having the largest distance among the four consecutive transitions both  by $\delta_W$ at the seven largest sampled $\tau$'s and by COOT. On the other hand,  smaller diffusion scales emphasize more local network organization, and hence the L1 to L2 discrepancy may reflect stronger changes in local connectivity, in line with this transition having the largest $\delta_W$ at the 12 smallest $\tau$'s. 

Another hypothesis, in terms of distribution shifts between the transitions, is that for smaller diffusion scales their resulting $\delta_W$ discrepancy is driven largely by the change in distribution supports, particularly affine changes in location and  spread;  see \Cref{fig:tau 0.1 OR hist,fig:tau 1 OR hist} in \Cref{appendix:suppfigs} for a qualitative illustration of this contrast: support differences are more apparent at the smaller scale $\tau=0.1$, whereas the supports largely overlap at the larger scale $\tau=1$. This may have contributed to L1 to L2 having larger raw $\delta_W$ than L2 to L3 at the 12 smallest $\tau$'s, while having a smaller $\delta_{W,cs}$.  However, the persistence of the L2 to L3 discrepancy  after centering and scaling (cf. column 5 of \Cref{tab:dev-curvature}) also suggests that this difference is not explained solely by shifts in the location or spread of the curvature distributions.

}

\subsection{Co-Optimal transport for measure hypernetworks}\label{subsec:coopt_hypernetworks}
\begin{figure}[!htbp]
    \centering
    
    \begin{subfigure}[b]{0.48\textwidth}
        \centering
        \includegraphics[width=\linewidth]{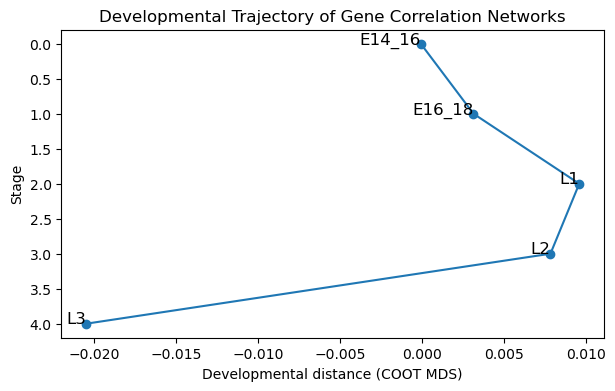}
        \caption{Developmental trajectory of \textit{Drosophila} across different stages.}
        \label{fig:development_trajectory}
    \end{subfigure}
    \hfill
    \begin{subfigure}[b]{0.48\textwidth}
        \centering
        \includegraphics[width=\linewidth]{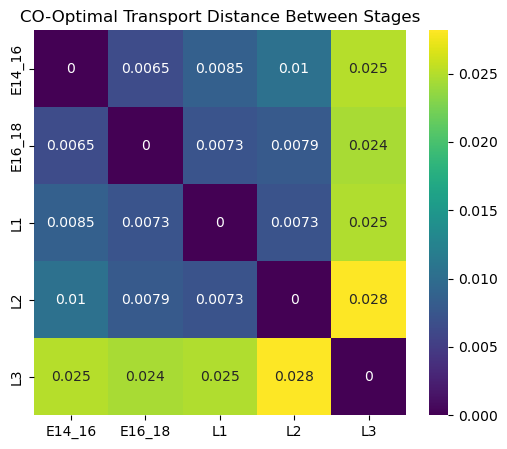}
        \caption{Pairwise COOT distance between developmental stages.}
        \label{fig:heatmap_coot}
    \end{subfigure}
    
    \caption{
    Comparison of developmental-stage relationships inferred from gene correlation networks. 
    (a) COOT-based multidimensional scaling (MDS) embedding of developmental stages (E14–16, E16–18, L1, L2, L3), illustrating the trajectory of network evolution over time. Early embryonic stages cluster closely, whereas later stages, particularly L3, show a larger shift in network structure. 
    (b) Heatmap of pairwise CO-Optimal Transport (CO-OT) distances between developmental stages. Lower values indicate greater similarity, revealing gradual transcriptional changes from E14–16 through L2 and a more pronounced divergence at L3.
    }
    \label{fig:trajectory_heatmap}
\end{figure}

The developmental trajectory in Figure~\ref{fig:development_trajectory} is obtained by applying multidimensional scaling (MDS) \cite{kruskal1964nonmetric} to the pairwise $\textrm{COOT}_p$ distances defined in \eqref{eq:COOT} with $p=2$ computed between stage-specific cell--gene interaction matrices, represented as measure hypernetworks. The resulting one-dimensional embedding summarizes the structural similarity of spatial transcriptomic organization across developmental stages. The close proximity of the late embryonic stages (E14--16 and E16--18) indicates highly similar cell--gene interaction patterns, suggesting relatively stable spatial and molecular organization during late embryogenesis. The gradual progression through the early larval stages (L1 and L2) reflects continuous remodeling of the joint cellular and molecular architecture. In contrast, the marked separation of the L3 stage indicates substantially greater structural divergence, consistent with extensive reorganization of cell--gene interactions accompanying the transition toward metamorphosis. Because COOT jointly aligns the cell and gene domains, the trajectory exhibits changes in the global organization of spatial transcriptomic profiles rather than simple differences in expression levels.

Figure~\ref{fig:heatmap_coot} presents the pairwise COOT distance matrix between stage-specific cell--gene interaction matrices. The relatively small distances among the late embryonic and early larval stages (E14--16, E16--18, L1, and L2; approximately 0.006--0.010) indicate that the overall spatial transcriptomic organization remains highly conserved throughout these developmental stages. In contrast, the substantially larger distances involving L3 (approximately 0.024--0.028) suggest a marked structural divergence from the earlier stages, pointing to a major reorganization of cell--gene interactions during late larval development. Together, the heatmap and MDS embedding reveal a developmental trajectory characterized by gradual structural changes from embryogenesis through the early larval stages, followed by a distinct transition at L3. 

As highlighted in \cite{wang2025drosophila}, the development of \textit{Drosophila} can be understood as a gradual process where changes in genes, cells, and tissue structures occur together over time. During embryonic development, tissue organization is established step by step, as cells receive signals that guide them into different roles and locations within the developing stage of the \textit{Drosophila}. As development progresses, cells become increasingly organized into specialized tissues and structures. By the end of embryogenesis, many early cell types and tissue structures have begun to form, providing a foundation for further development during larval stages. Early larval development continues the refinement and maturation of these tissues as cells adjust their functions during growth. In contrast, the later larval stage, particularly L3, represents a major transition as the organism prepares for metamorphosis. During this stage, cells undergo substantial changes in their activity and developmental state, accompanied by a widespread shift in gene-expression profiles. Consistent with this developmental pattern, the COOT-derived embedding of the stage-specific cell–gene heatmap in Figure~\ref{fig:heatmap_coot} shows a gradual progression through embryonic and early larval stages, while L3 forms a clearly separated group, reflecting the major developmental changes occurring before metamorphosis.

\section{Discussion}\label{sec:discussion}

In this work, we developed a geometric framework to study developmental trajectories in the \textit{Drosophila melanogaster} spatial transcriptomic dataset by representing gene expression patterns as spatially and temporally varying co-expression networks. Using tools from optimal transport, including Gromov--Wasserstein (GW) geometry and Ollivier--Ricci (OR) curvature, we investigated how network structure changes during development.

Our experiments lead to three main conclusions. First, OR curvature distributions can distinguish between gene expression networks collected from different spatial locations and at different time points. { The dynamic OR curvature analysis further provides a scale-dependent view of developmental differences.} This supports prior work in this area regarding the efficacy of OR curvature in recovering biological information from gene expression networks while also incorporating both temporal and spatial components. Second, even in the absence of intermediate time points, curvature distributions along GW geodesics recover developmental trends observed in real data. This suggests that GW geodesics provide meaningful interpolations between gene expression networks and can reveal aspects of developmental progression when temporal measurements are incomplete. Third, by examining developmental changes using measure hypernetworks and Co-Optimal Transport (COOT), the resulting COOT embedding exhibited a clear progression from embryonic to larval stages, with relatively small differences among early stages and a distinct separation of the L3 stage. By jointly aligning genes and cells, COOT retains higher-order relationships that are not represented in pairwise graph comparisons, allowing it to detect broader structural changes associated with late larval development.

Several limitations should be noted. First, our gene co-expression networks rely on correlation-based thresholding following \cite{Graph_Curvature_Sandhu2015}, making network topology sensitive to choices such as correlation thresholds, gene selection, and treatment of negative correlations. Future work should evaluate alternative network inference strategies and biologically informed interaction networks. Second, Gromov--Wasserstein optimization remains computationally challenging for large-scale transcriptomic datasets, motivating further algorithmic improvements. 

Future extensions include incorporating biologically informed edge weights from regulatory interactions, ligand--receptor communication, and multi-omics data \cite{bunne2024optimal}, exploring the OR curvature notion for hypergraphs \cite{murgas2022hypergraph}, and exploring unbalanced such as conic variants of GW and COOT to accommodate changes in cell populations and gene regulations \cite{oliver2025conic}. More broadly, combining curvature-based descriptors with higher-order transport representations may provide a unified framework for linking molecular interactions, cellular organization, and developmental dynamics.

Overall, this study demonstrates that GW-based geometric methods provide a natural framework for analyzing spatiotemporal transcriptomic data beyond conventional distribution alignment. The agreement between empirical developmental patterns, GW interpolations, curvature signatures, and COOT embeddings highlights that OT geometry encodes biologically meaningful information across multiple scales, offering new opportunities for quantitative analysis of developmental and cellular processes.

\section*{Data and code availability}
The raw data analyzed in this study are publicly available from \url{https://www.cell.com/cell/fulltext/S0092-8674(25)00629-4}. The processed data and code used to reproduce the computational results and figures are available through GitHub at  \url{https://github.com/TuyentdTran/A-Unified-Geometric-Framework-for-Developmental-Analysis-of-Spatial-Transcriptomic-Data}.

\section*{Acknowledgments}
The authors acknowledge the support of the National Science Foundation (NSF) Grant No. DMS-2520375 via the 2025 Research Collaboration Workshop in the Science of Data and Mathematics (WiSDM) held at the University of North Carolina at Chapel Hill on August 4-8, 2025. The authors also acknowledge Jannatul Chhoa as a participant in the project during the August 2025 WiSDM workshop. K. Hohmeier is supported by the NSF Graduate Research Fellowship Program under Grant No. DGE-2040435.  MC. Antony Oliver acknowledges the receipt of funding obtained from the Health Data Research UK-The Alan Turing Institute 
Wellcome (218529/Z/19/Z) and the Cambridge Trust scholarship from 
the Commonwealth European and International Trust (CCEIT). C. Moosm\"{u}ller and S. Li were partially supported by NSF award DMS-2410140. The funders had no role in study design, data collection and analysis, decision 
to publish, or preparation of the manuscript.

\bibliographystyle{plainnat}
\bibliography{refs.bib}

\appendix

\section{Supplementary Figures and Tables }\label{appendix:suppfigs}
\begin{figure}[!htbp]
    \centering
    \includegraphics[scale=0.45]{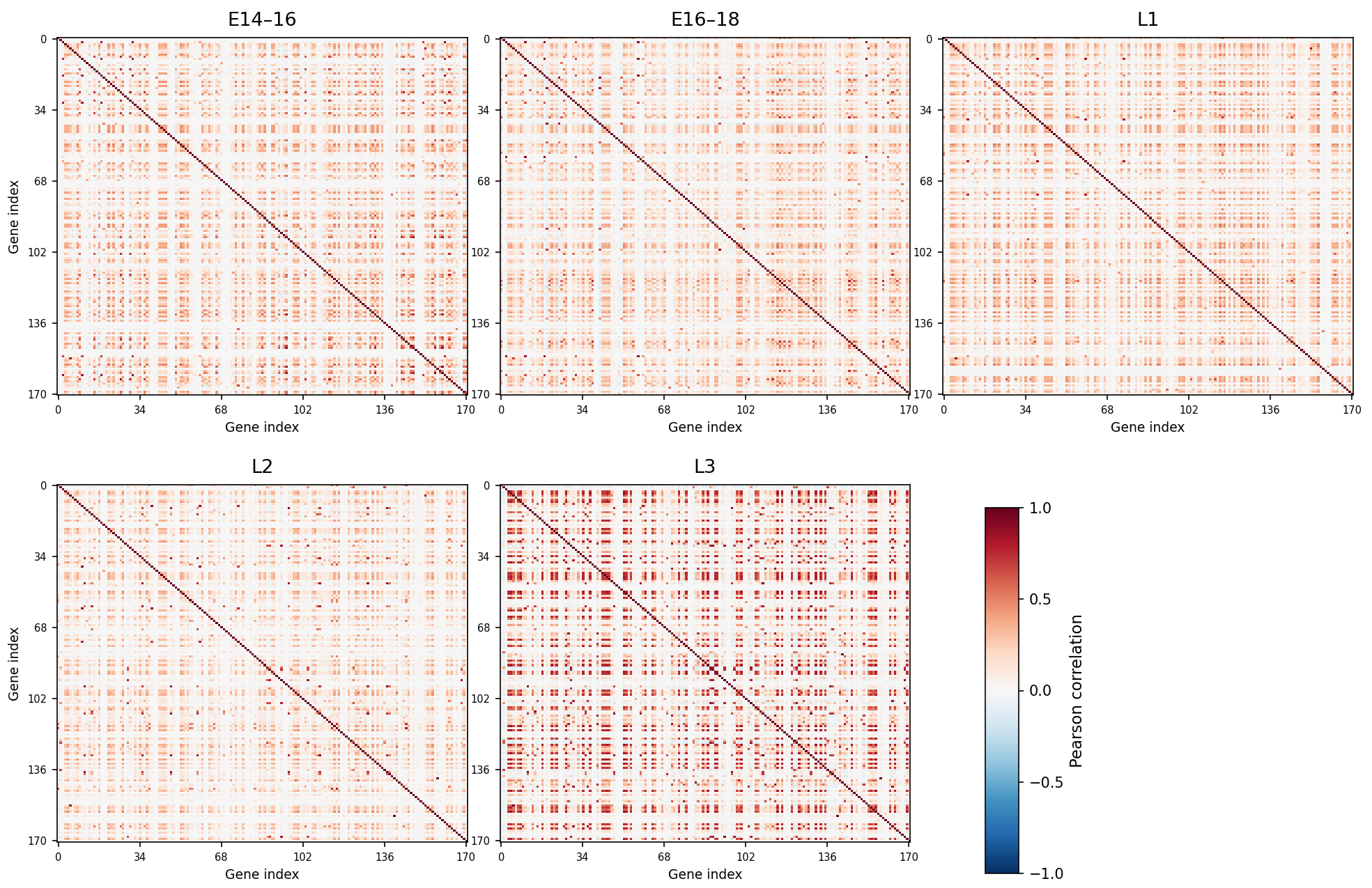}
    \caption{Heatmap of pairwise gene correlations across all \textit{Drosophila} developmental stages (E14-16, E16-18, L1, L2, L3), using a union of common genes across all five life stages. Other than large changes in correlation values at L3, the correlation values as well as the structure of the gene relationships at each time step appear visually similar.}
    \label{fig:correlations}
\end{figure}
\begin{figure}[!htpb]
    \centering
    \includegraphics[scale=0.4]{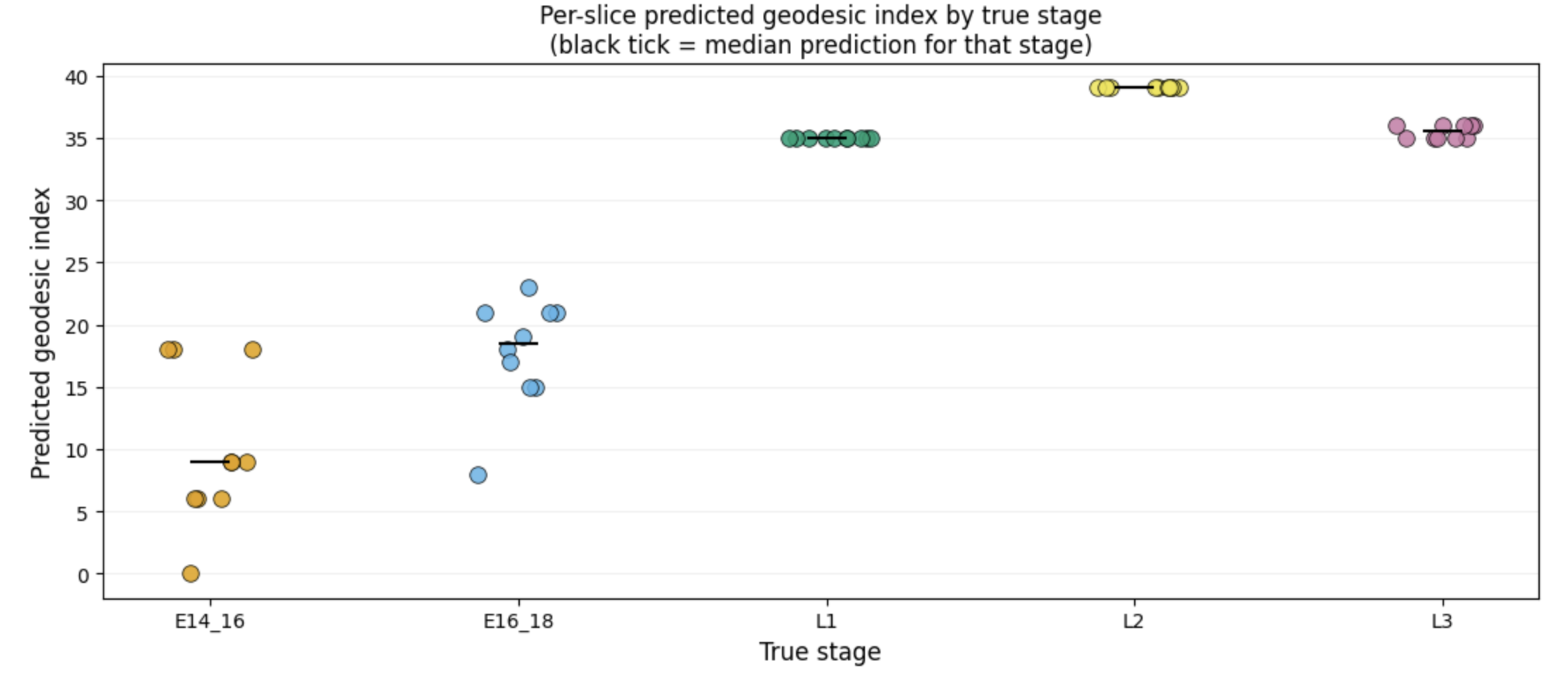}
    \caption{Rank ordering results for the mean method.}
\label{fig:meangeoclass}
\end{figure}

\begin{figure}[!htpb]
    \centering
    \begin{subfigure}[b]{0.48\textwidth}
        \centering
        \includegraphics[scale=0.45]{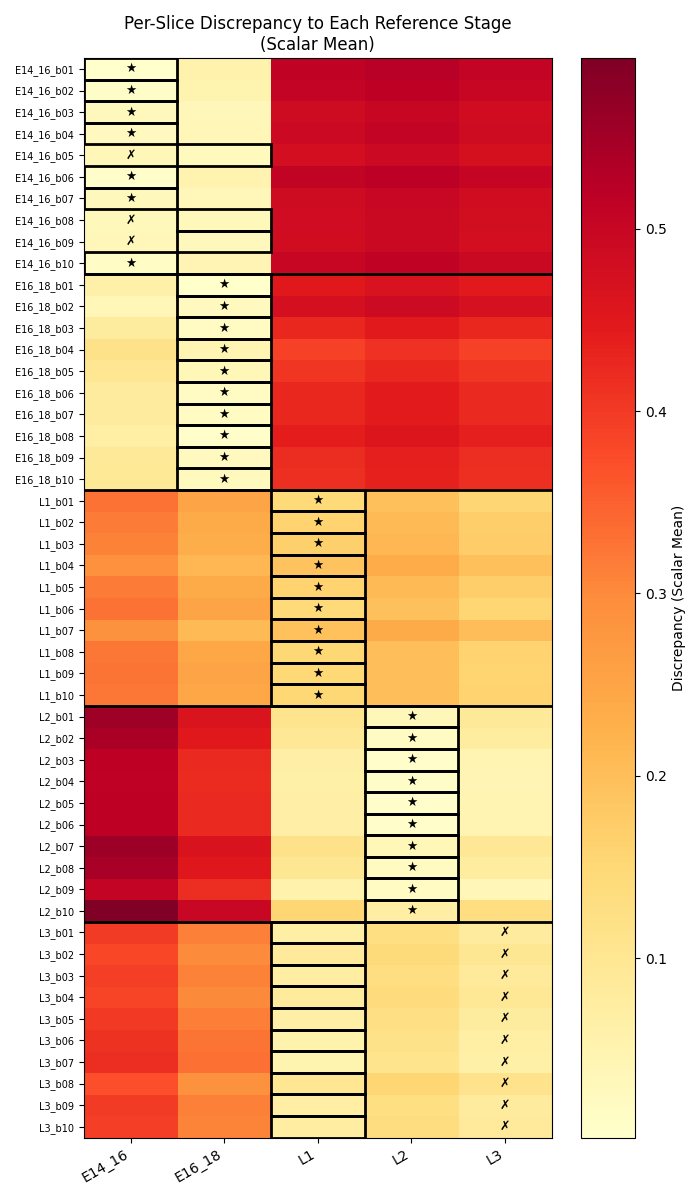}
        \caption{Mean.}
        \label{subfig:meandiscrepplot}
    \end{subfigure}
    \begin{subfigure}[b]{0.48\textwidth}
        \centering
        \includegraphics[scale=0.45]{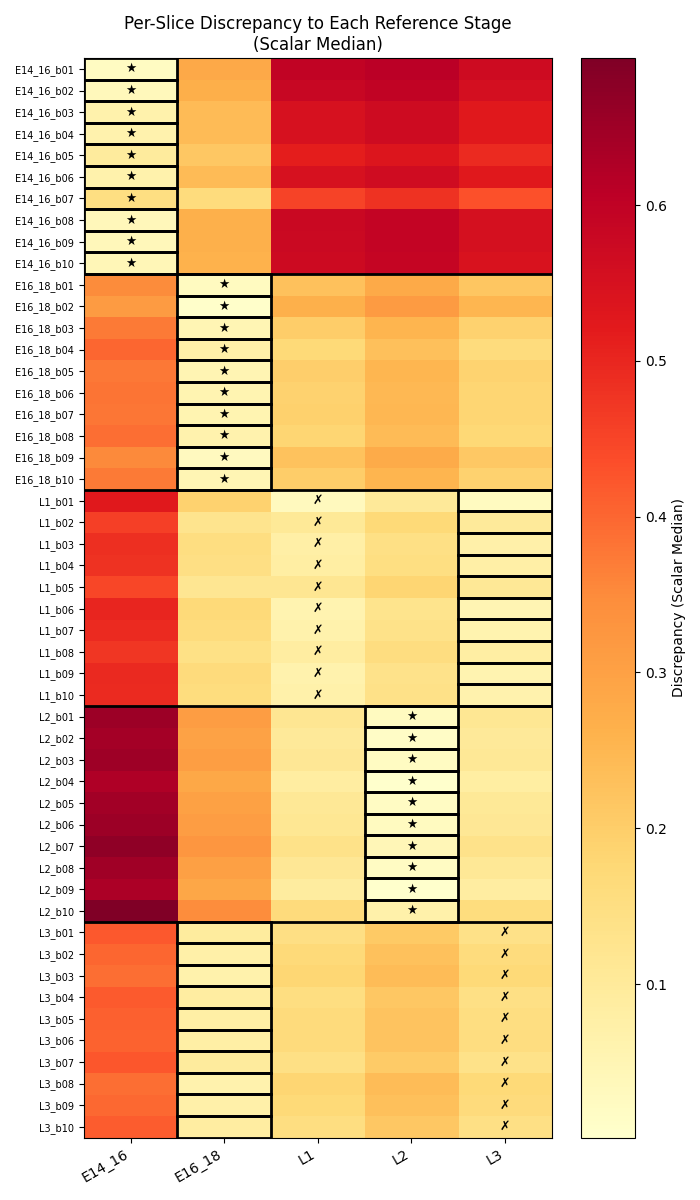}
        \caption{Median.}
        \label{subfig:meddiscrepancy}
    \end{subfigure}
    \caption{Plot of per-slice, per-reference-stage discrepancy values for the mean and median classification methods. Correctly-classified slices are marked with a star symbol and a black box, while incorrectly-classified slices are denoted with an x symbol, with a black box around the class into which it was incorrectly classified.}
    \label{fig:meanmeddiscrep}
\end{figure}

\begin{figure}[!htbp]
	\centering
	\begin{subfigure}{0.45\textwidth}
		\includegraphics[width=\linewidth]{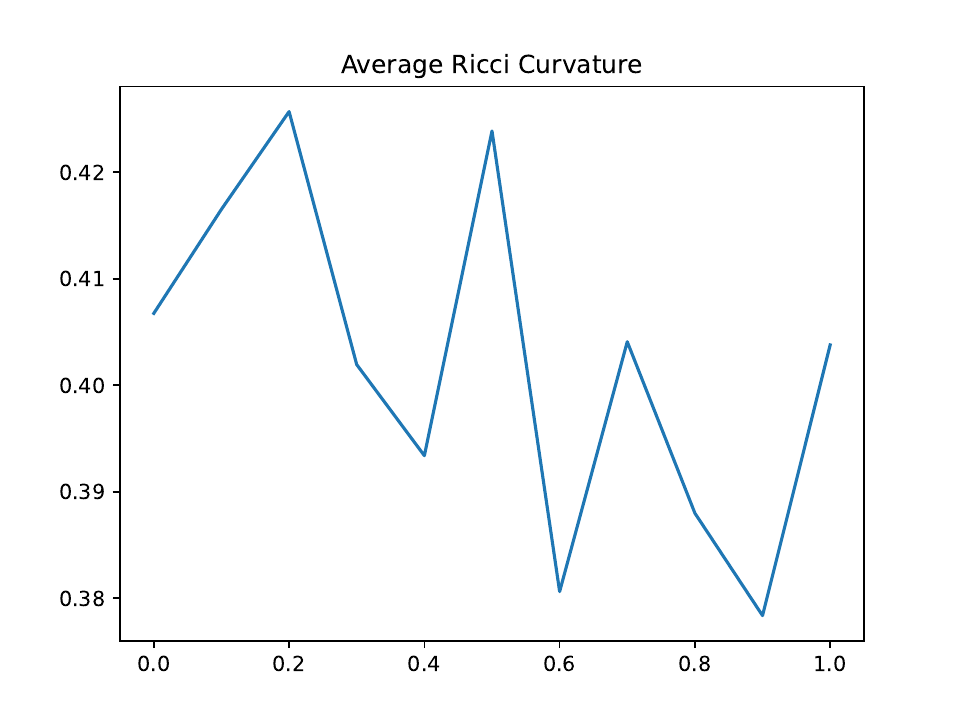}
		\caption{E14-16 to E16-18}
	\end{subfigure}
	\begin{subfigure}{0.45\textwidth}
		\includegraphics[width=\linewidth]{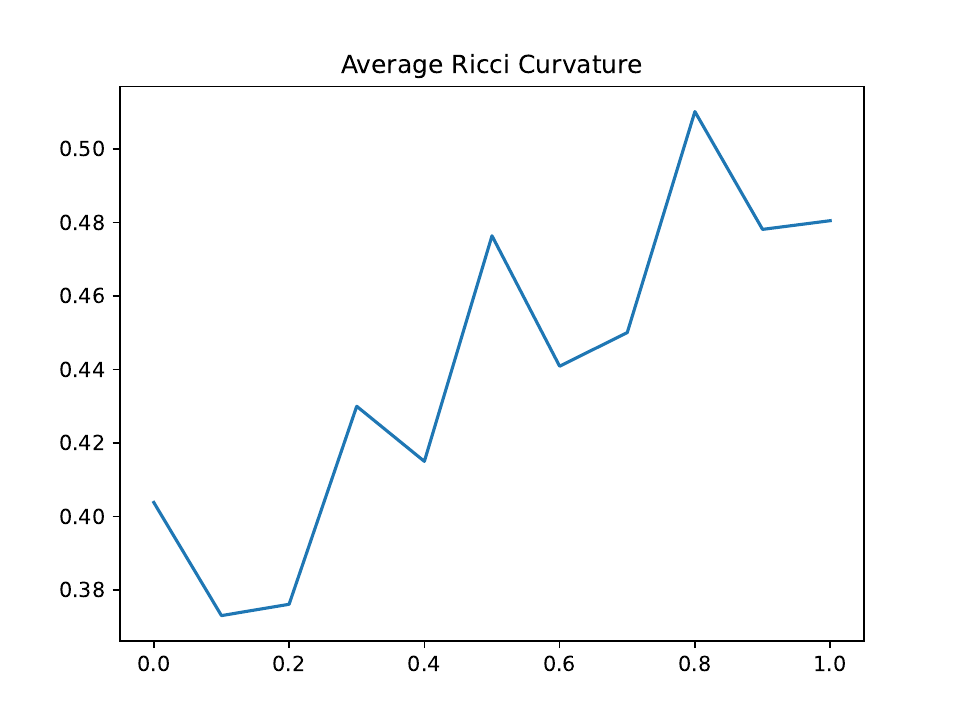}
		\caption{E16-18 to Larvae 1}
	\end{subfigure}
	\begin{subfigure}{0.45\textwidth}
		\includegraphics[width=\linewidth]{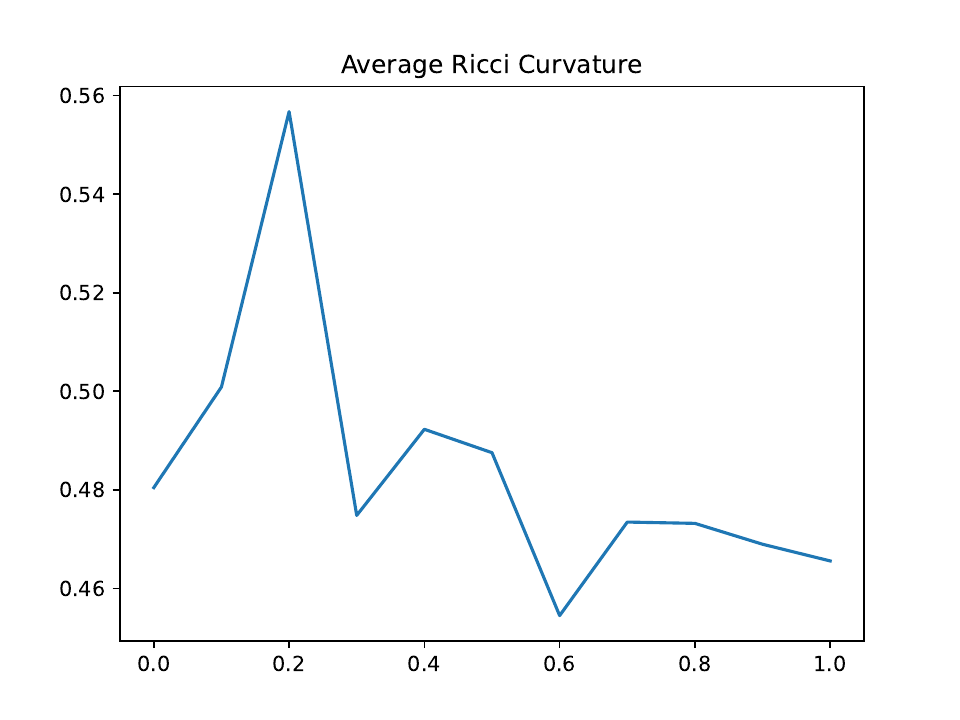}
		\caption{Larvae 1 to Larvae 2}
	\end{subfigure}
	\begin{subfigure}{0.45\textwidth}
		\includegraphics[width=\linewidth]{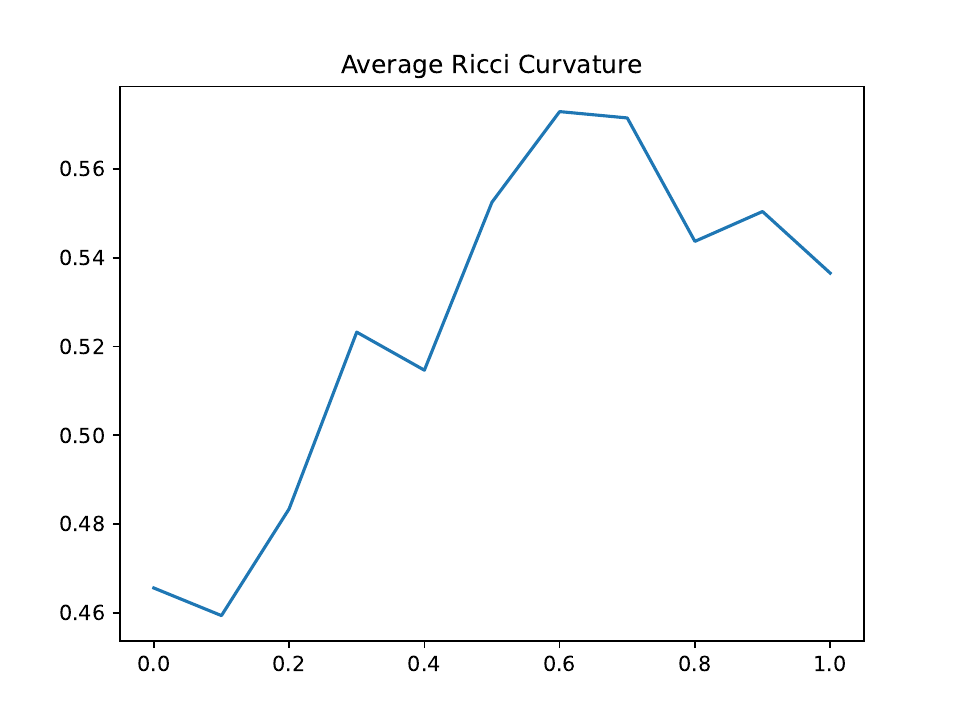}
		\caption{Larvae 2 to Larvae 3}
	\end{subfigure}

	\caption{Average OR curvatures for the geodesics between different stages of the fruit fly life spans as detailed in \cref{subsec: geodesics and curvature}. A gene network threshold of \(\mu =0.2\) was used. }
	\label{fig:fruit fly 0.2 graph threshold}
\end{figure}

\begin{table}[H]
\centering
\caption{Values for $\widetilde{\kappa}_{G_B^{w,t_s}}$ and $\widetilde{\kappa}_{G_s}$ in \cref{subsec: geodesics and curvature}.}
\label{tab: average OR curvature values}
\begin{tabular}{lrrrrr}
\toprule
 & $s=1$ & $s=2$ & $s=3$ & $s=4$ & $s=5$ \\
\midrule
$\widetilde{\kappa}_{G_s}$ & 0.4067 & 0.4038 & 0.4805 & 0.4656 & 0.5365 \\
$\widetilde{\kappa}_{G_B^{w,t_s}}$ & 0.4067 & 0.4680 & 0.4558 & 0.5199 & 0.5365 \\
\bottomrule
\end{tabular}
\end{table}

\begin{figure}[!htpb]
    \centering
    \includegraphics[scale=0.4]{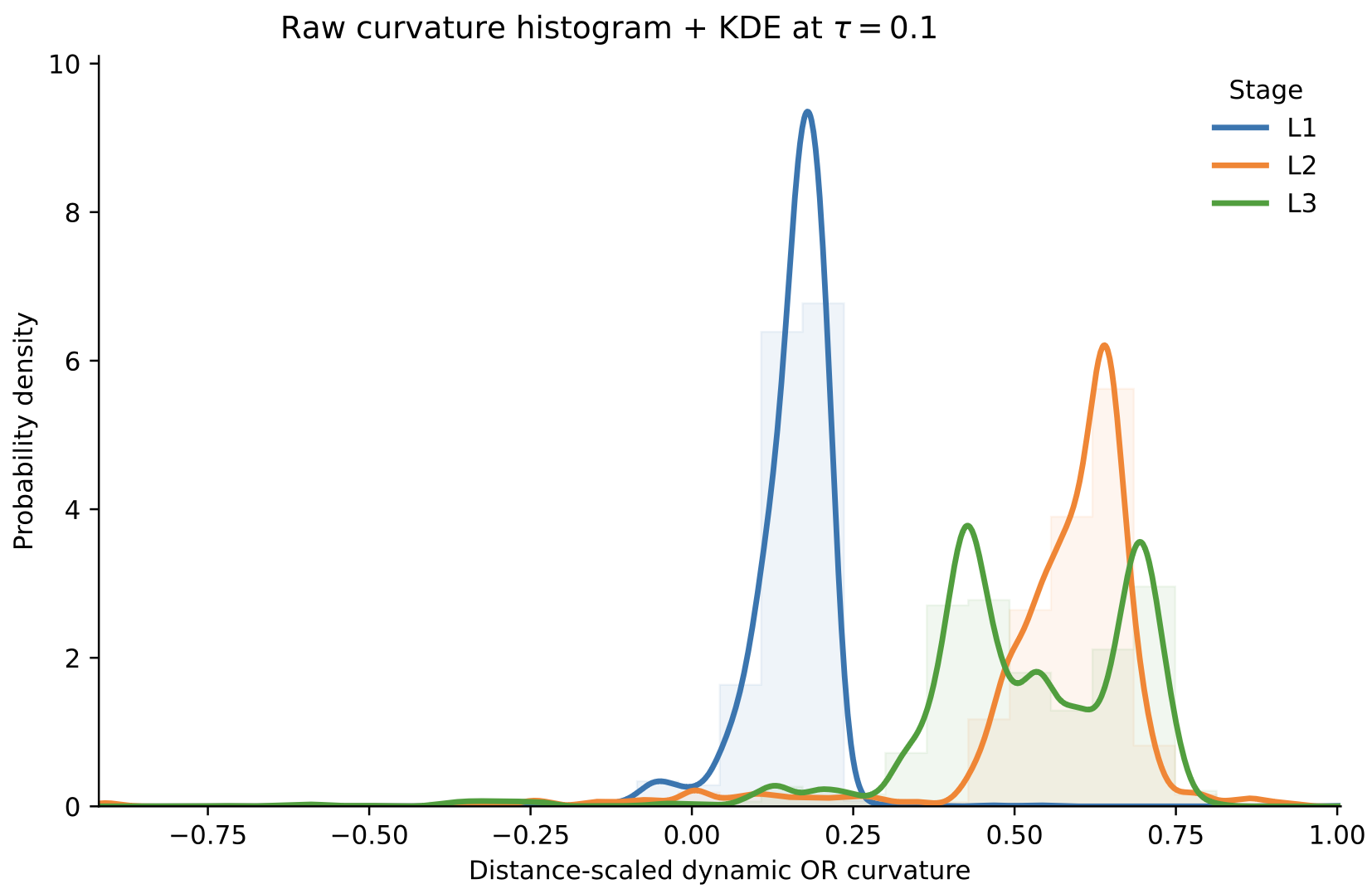}
    \caption{Histograms of distance-scaled dynamic OR curvature values $\kappa^{\mathrm{ds}}_{x,y}(\tau)$ with $\tau = 0.1$ for L1, L2, and L3. The kernel density estimates (KDEs) are included for visualization. }
\label{fig:tau 0.1 OR hist}
\end{figure}

\begin{figure}[!htpb]
    \centering
    \includegraphics[scale=0.4]{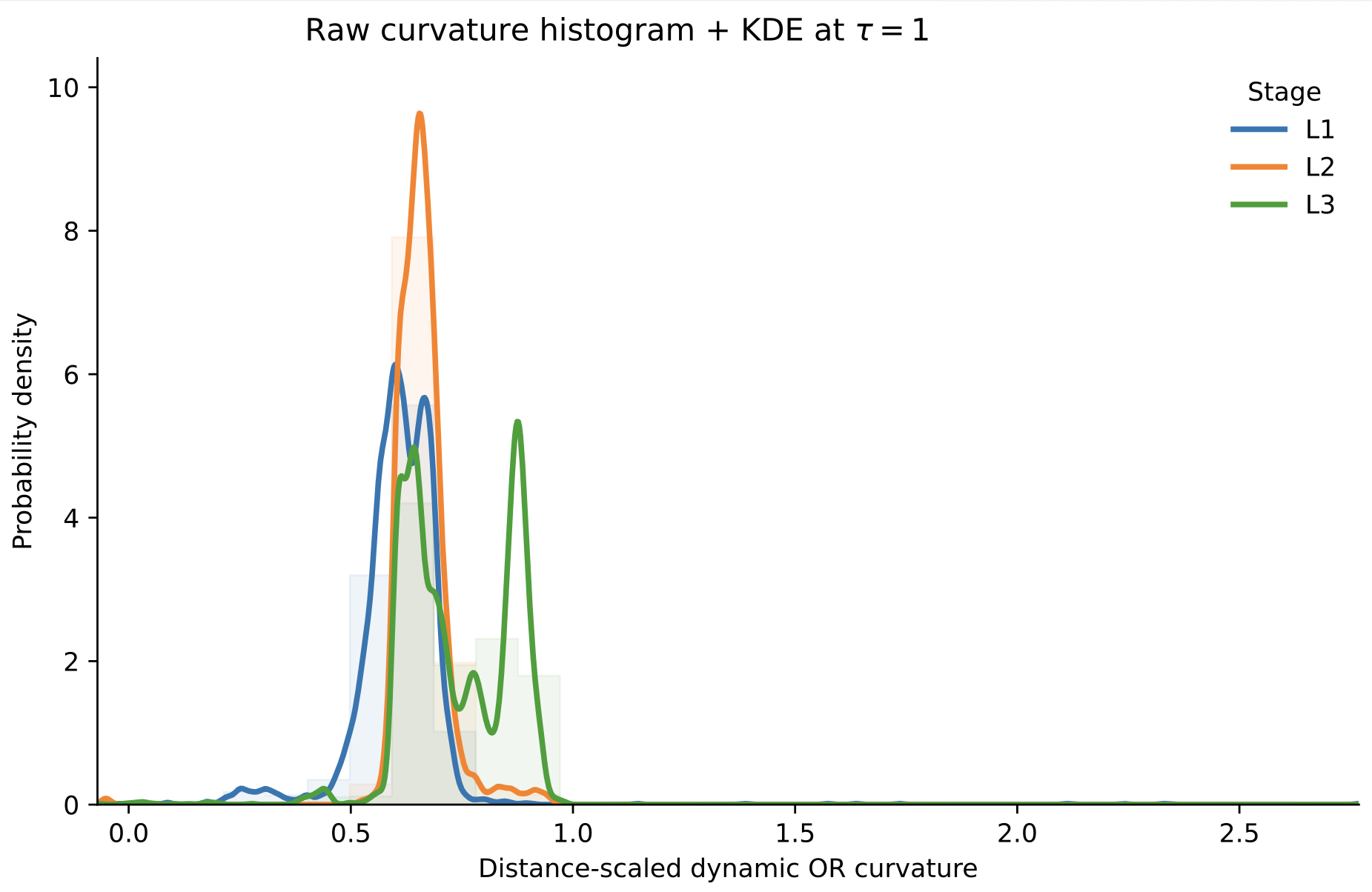}
    \caption{Histograms of distance-scaled dynamic OR curvature values $\kappa^{\mathrm{ds}}_{x,y}(\tau)$  with $\tau = 1$ for L1, L2, and L3. The kernel density estimates (KDEs) are included for visualization.}
\label{fig:tau 1 OR hist}
\end{figure}

\newpage
\section{Supplementary Details}
\subsection{Chance baseline for relaxed sign alignment}\label{appendix:relaxed baseline}

Let $f$ and $g$ be the scalar trends defined in \Cref{subsection: trend analysis with OR curvature}, with increments
\begin{equation*}
\Delta f_i=f(t_{i+1})-f(t_i),
\qquad
\Delta g_i=g(t_{i+1})-g(t_i),
\qquad
i=1,\ldots,n-1.
\end{equation*}

\begin{Lem}\label{lm: relaxed chance baseline}
Let $\Delta f$ and $\Delta g$ be independent random increments corresponding to two scalar trends $f$ and $g$. Assume that, conditional on having magnitude at least $\epsilon$, each increment is equally likely to be positive or negative. Define
\begin{equation*}
p_{f,\epsilon}
=
\mathbb{P}\bigl(|\Delta f|<\epsilon\bigr),
\qquad
p_{g,\epsilon}
=
\mathbb{P}\bigl(|\Delta g|<\epsilon\bigr).
\end{equation*}
Then the chance baseline for the relaxed sign-alignment score is
\begin{equation}
\label{eq: relaxed chance baseline}
C_\epsilon
=
1-\frac{1}{2}
\bigl(1-p_{f,\epsilon}\bigr)
\bigl(1-p_{g,\epsilon}\bigr).
\end{equation}
\end{Lem}

\begin{proof}
Under the relaxed criterion, disagreement occurs only when both increments have magnitude at least $\epsilon$ and have opposite signs. By independence,
\begin{equation*}
\mathbb{P}\bigl(
|\Delta f|\geq\epsilon,\,
|\Delta g|\geq\epsilon
\bigr)
=
\bigl(1-p_{f,\epsilon}\bigr)
\bigl(1-p_{g,\epsilon}\bigr).
\end{equation*}
Conditional on both increments having magnitude at least $\epsilon$, their signs are independent and equally likely to be positive or negative, so the probability of opposite signs is $\frac{1}{2}$. 
Hence, the probability of relaxed disagreement is
\begin{equation*}
\frac{1}{2}
\bigl(1-p_{f,\epsilon}\bigr)
\bigl(1-p_{g,\epsilon}\bigr).
\end{equation*}
Taking its complement gives \eqref{eq: relaxed chance baseline}.
\end{proof}

The small-change probabilities $p_{f,\epsilon}, p_{g,\epsilon}$ are estimated from the observed increments by
\begin{equation*}
\widehat{p}_{f,\epsilon}
=
\frac{1}{n-1}
\sum_{i=1}^{n-1}
\mathbb{I}\bigl(|\Delta f_i|<\epsilon\bigr),
\qquad
\widehat{p}_{g,\epsilon}
=
\frac{1}{n-1}
\sum_{i=1}^{n-1}
\mathbb{I}\bigl(|\Delta g_i|<\epsilon\bigr).
\end{equation*}
The corresponding empirical relaxed chance baseline is
\begin{equation}
\label{eq: empirical relaxed chance baseline}
\widehat{C}_\epsilon
=
1-\frac{1}{2}
\bigl(1-\widehat{p}_{f,\epsilon}\bigr)
\bigl(1-\widehat{p}_{g,\epsilon}\bigr).
\end{equation}

\subsection{Additional Data Preprocessing}\label{sec:add_data_prepos}

In order to construct the relevant data for classification we consider a bootstrapped approach motivated by \cite{colby2018improving}. In particular, for each developmental stage, we generate 10 different correlation matrices using different random instances of the HGV genes across the slices at that stages. We reiterate the union approach of considering the same genes across the different slices for each sample which is analogous to the description in Section \ref{sec:data preprocessing}. The only difference here is that we introduce additional stochasticity for testing the robustness of our methods during sample generation through bootstrapping. Specifically, for each slice, we randomly select 50 top genes from 11703 highly variable genes (HVGs) identified across all slices. This is done in a manner that the different correlation matrices can serve as induced data points to understand the temporal ordering and classification results in Section \ref{subsec:classification}.

\end{document}